\documentclass{article}

     \usepackage[final]{neurips_2019}

\usepackage[utf8]{inputenc} 
\usepackage[T1]{fontenc}    
\usepackage{hyperref}       
\usepackage{url}            
\usepackage{booktabs}       
\usepackage{amsfonts}       
\usepackage{nicefrac}       
\usepackage{microtype}      
\usepackage{amsmath}

\usepackage{natbib}
\usepackage{bm}
\usepackage{amstext}
\usepackage{amsthm}
\usepackage{amssymb}
\usepackage{natbib}
\usepackage{enumitem}
\usepackage{dsfont}
\newcommand{\ind}{\mathds{1}}
\theoremstyle{plain}
\newtheorem{thm}{\protect Theorem}
\theoremstyle{plain}
\newtheorem{lem}[thm]{\protect Lemma}
\theoremstyle{plain}

\usepackage{multirow}
\usepackage{graphicx}
\usepackage{subcaption}
\usepackage[table]{xcolor}
\newcommand{\tripleBar}{\|} 
\usepackage[ruled,vlined,linesnumbered,boxed]{algorithm2e}

\title{Missing Not at Random in Matrix Completion: \\ The Effectiveness of Estimating Missingness Probabilities Under a Low Nuclear Norm Assumption}

\author{%
  Wei Ma$^*$\qquad George H.~Chen\thanks{Equal contribution \newline\indent~$\hspace{0.29em}$ Code available at \url{https://github.com/georgehc/mnar_mc}} \\
  Carnegie Mellon University \\
  Pittsburgh, PA 15213 \\
  \texttt{\{weima,georgechen\}@cmu.edu} \\
}

\begin{document}

\maketitle

\begin{abstract}
Matrix completion is often applied to data with entries missing not at random (MNAR). For example, consider a recommendation system where users tend to only reveal ratings for items they like. In this case, a matrix completion method that relies on entries being revealed at uniformly sampled row and column indices can yield overly optimistic predictions of unseen user ratings. Recently, various papers have shown that we can reduce this bias in MNAR matrix completion if we know the probabilities of different matrix entries being missing. These probabilities are typically modeled using logistic regression or naive Bayes, which make strong assumptions and lack guarantees on the accuracy of the estimated probabilities. In this paper, we suggest a simple approach to estimating these probabilities that avoids these shortcomings. Our approach follows from the observation that missingness patterns in real data often exhibit low nuclear norm structure. We can then estimate the missingness probabilities by feeding the (always fully-observed) binary matrix specifying which entries are revealed or missing to an existing nuclear-norm-constrained matrix completion algorithm by Davenport et al.~[2014]. Thus, we tackle MNAR matrix completion by solving a different matrix completion problem first that recovers missingness probabilities. We establish finite-sample error bounds for how accurate these probability estimates are and how well these estimates debias standard matrix completion losses for the original matrix to be completed. Our experiments show that the proposed debiasing strategy can improve a variety of existing matrix completion algorithms, and achieves downstream matrix completion accuracy at least as good as logistic regression and naive Bayes debiasing baselines that require additional auxiliary information.
\end{abstract}

\section{Introduction}

Many modern applications involve partially observed matrices where entries are missing not at random (MNAR). For example, in restaurant recommendation, consider a ratings matrix~$X\in(\mathbb{R}\cup\{\star\})^{m\times n}$ where rows index users and columns index restaurants, and the entries of the matrix correspond to user-supplied restaurant ratings or ``$\star$'' to indicate ``missing''. A user who is never in London is unlikely to go to and subsequently rate London restaurants, and a user who is vegan is unlikely to go to and rate restaurants that focus exclusively on meat. In particular, the entries in the ratings matrix are not revealed uniformly at random. As another example, in a health care context, the partially observed matrix $X$ could instead have rows index patients and columns index medically relevant measurements such as latest readings from lab tests. Which measurements are taken is not uniform at random and involve, for instance, what diseases the patients have. 
Matrix completion can be used in both examples: predicting missing entries in the recommendation context, or imputing missing features before possibly using the imputed feature vectors in a downstream prediction task. 

The vast majority of existing theory on matrix completion assume that entries are revealed with the same probability independently (e.g., \citet{candes2009exact,cai2010singular,keshavan2010matrix1,keshavan2010matrix2,recht2011simpler,chatterjee2015matrix,song2016blind}). Recent approaches to handling entries being revealed with nonuniform probabilities have shown that estimating what these entry revelation probabilities can substantially improve matrix completion accuracy in recommendation data \citep{liang2016modeling,schnabel2016recommendations,wang2018modeling,wang2018collaborative,wang2019doubly}. Specifically, these methods all involve estimating the matrix $P\in[0,1]^{m\times n}$, where $P_{u,i}$ is the probability of entry $(u,i)$ being revealed for the partially observed matrix $X$. We refer to this matrix $P$ as the \emph{propensity score matrix}. By knowing (or having a good estimate of) $P$, we can debias a variety of existing matrix completion methods that do not account for MNAR entries \citep{schnabel2016recommendations}.


In this paper, we focus on the problem of estimating propensity score matrix $P$ and examine how error in estimating $P$ impacts downstream matrix completion accuracy. Existing work \citep{liang2016modeling,schnabel2016recommendations,wang2018collaborative,wang2019doubly} typically models entries of $P$ as outputs of a simple predictor such as logistic regression or naive Bayes. In the generative modeling work by \citet{liang2016modeling} and \citet{wang2018collaborative}, $P$ is estimated as part of a larger Bayesian model, whereas in the work by \citet{schnabel2016recommendations} and \citet{wang2019doubly} that debias matrix completion via inverse probability weighting (e.g., \citet{imbens2015causal}), $P$ is estimated as a pre-processing~step.

Rather than specifying parametric models for $P$, we instead hypothesize that in real data, $P$ often has a particular low nuclear norm structure (precise details are given in Assumptions A1 and A2 in Section \ref{sec:theory}; special cases include $P$ being low rank or having clustering structure in rows/columns). Thus, with enough rows and columns in the partially observed matrix $X$, we should be able to recover~$P$ from the missingness mask matrix $M\in\{0,1\}^{m\times n}$, where $M_{u,i}=\ind\{X_{u,i} \ne \star\}$. For example, for two real datasets \texttt{Coat} \citep{schnabel2016recommendations} and \texttt{MovieLens-100k} \citep{harper2016movielens}, their missingness mask matrices $M$ (note that these are always fully-observed) have block structure, as shown in Figure~\ref{fig:block-structure}, suggesting that they are well-modeled as being generated from a low rank $P$; with values of $P$ bounded away from 0, such a low rank $P$ is a special case of the general low nuclear norm structure we consider. In fact, the low rank missingness patterns of \texttt{Coat} and \texttt{MovieLens-100k} can be explained by topic modeling, as we illustrate in Appendix~\ref{sec:topic-modeling}.

\begin{figure}[t]
  \centering
  \begin{subfigure}[b]{0.45\linewidth}
    \centering\includegraphics[width=\linewidth]{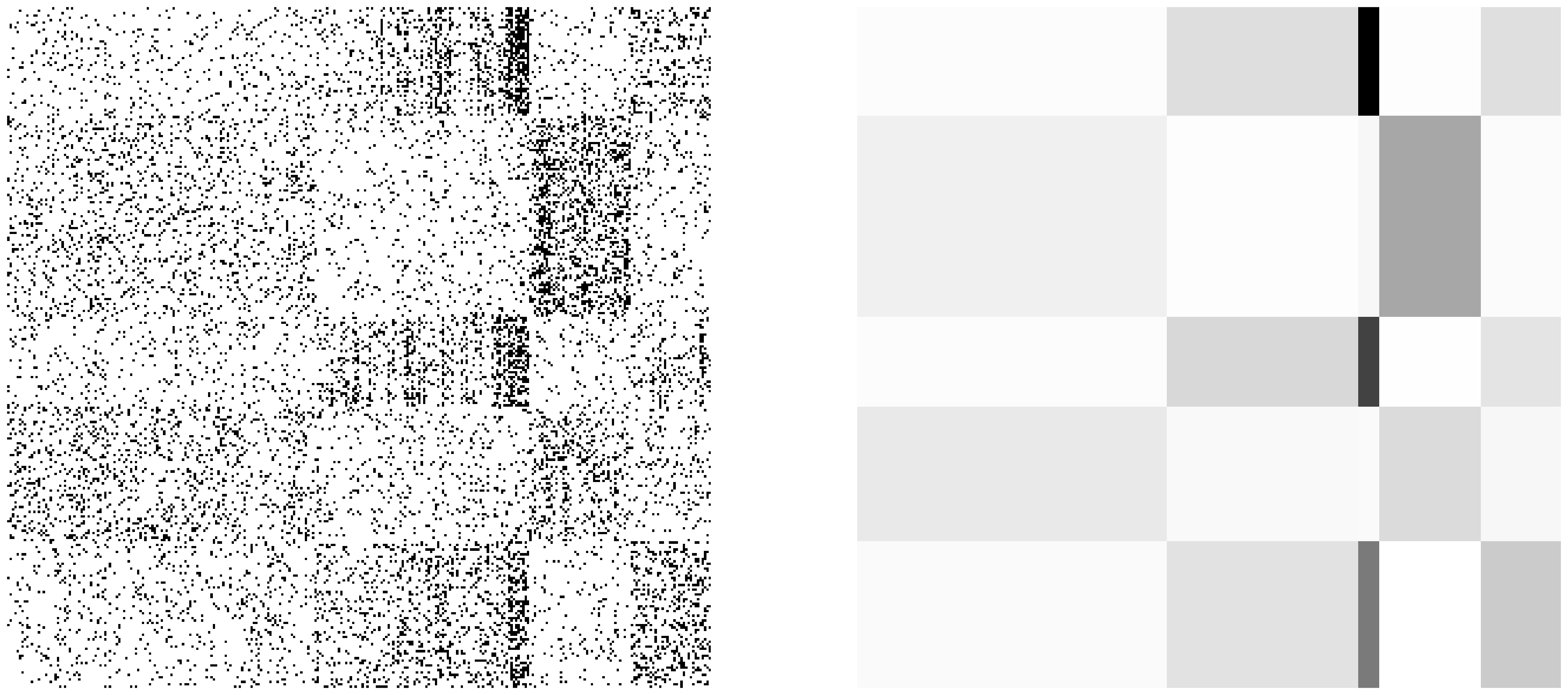}
    \caption{\texttt{Coat}}
  \end{subfigure}%
  $\qquad$
  \begin{subfigure}[b]{0.45\linewidth}
    \centering\includegraphics[width=\linewidth]{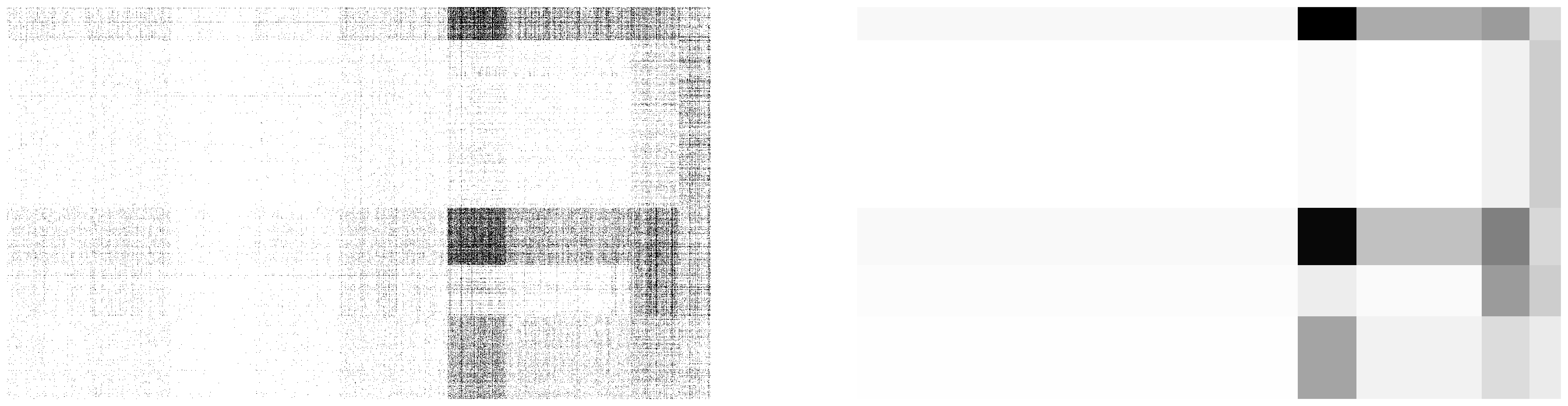}
    \caption{\texttt{MovieLens-100k}}
  \end{subfigure}
  \caption{Missingness mask matrices (with rows indexing users and columns indexing items) exhibit low-rank block structure for the (a) \texttt{Coat} and (b) \texttt{MovieLens-100k} datasets. Black indicates an entry being revealed.
  For each dataset, we show the missingness mask matrix on the left and the corresponding block structure identified using spectral biclustering \citep{kluger2003spectral} on the right; rows and columns have been rearranged based on the biclustering result.}
  \label{fig:block-structure}
  \vspace{-1em}
\end{figure}

We can recover propensity score matrix $P$ from missingness matrix $M$ using the existing 1-bit matrix completion algorithm by \citet{davenport20141}. This algorithm, which we refer to as \textsc{1bitMC}, solves a convex program that amounts to nuclear-norm-constrained maximum likelihood estimation. We remark that Davenport et al.~developed their algorithm for matrix completion where entries are missing independently with the same probability and the revealed ratings are binary. We intentionally apply their algorithm instead to the matrix $M$ of binary values for which there are no missing entries. Thus, rather than completing a matrix, we use \textsc{1bitMC} to denoise $M$ to produce propensity score matrix estimate $\widehat{P}$. Then we use $\widehat{P}$ to help debias the actual matrix completion problem that we care about: completing the original partially observed matrix $X$.

Our contributions are as follows:
\begin{itemize}[leftmargin=1.5em,topsep=0pt,partopsep=1ex,parsep=1ex]
\item We establish finite-sample bounds on the mean-squared error (MSE) for estimating propensity score matrix $P$ using \textsc{1bitMC} and also on its debiasing effect for standard MSE or mean absolute error (MAE) matrix completion losses (the debiasing is via weighting entries inversely by their estimated propensity scores).
\item We empirically examine the effectiveness of using \textsc{1bitMC} to estimate propensity score matrix~$P$ compared to logistic regression or naive Bayes baselines. In particular, we use the estimated propensity scores from these three methods to debias a variety of matrix completion algorithms, where we find that \textsc{1bitMC} typically yields downstream matrix completion accuracy as good as or better than the other two methods. The \textsc{1bitMC}-debiased variants of matrix completion algorithms often do better than their original unmodified counterparts and can outperform some existing matrix completion algorithms that handle nonuniformly sampled data.
\end{itemize}

\section{Model and Algorithms}

\begingroup
\setlength\abovedisplayskip{2pt}
\setlength\belowdisplayskip{2pt}
\setlength\abovedisplayshortskip{1pt}
\setlength\belowdisplayshortskip{1pt}

\textbf{Model.}
Consider a signal matrix $S\in\mathbb{R}^{m\times n}$, a noise matrix $W\in\mathbb{R}^{m\times n}$, and a propensity score matrix $P\in[0,1]^{m\times n}$. All three of these matrices are unknown. We observe the matrix $X\in(\mathbb{R}\cup\{\star\})^{m\times n}$, where $X_{u,i}=S_{u,i}+W_{u,i}$ with probability $P_{u,i}$, independent of everything else; otherwise $X_{u,i}=\star$, indicating that the entry is missing. We denote $\Omega$ to be the set of entries that are revealed (i.e., $\Omega=\{(u,i): u\in[m],i\in[n]\text{ s.t.~}X_{u,i}\ne\star\}$), and we denote $X^*:=S+W$ to be the noise-corrupted data if we had observed all the entries. Matrix completion aims to estimate~$S$ given $X$, exploiting some structural assumption on $S$ (e.g., low nuclear norm, low rank, a latent variable model).

\textbf{Debiasing matrix completion with inverse propensity scoring.} Suppose we want to estimate $S$ with low mean squared error (MSE).
If no entries are missing so that we directly observe $X^*$, then the MSE of an estimate $\widehat{S}$ of $S$ is
\[
L_{\text{full MSE}}(\widehat{S}):=\frac{1}{mn}\sum_{u=1}^m\sum_{i=1}^n (\widehat{S}_{u,i}-X^*_{u,i})^{2}.
\]
However, we actually observe $X$ which in general has missing entries. The standard approach is to instead use the observed MSE:
\[
L_{\text{MSE}}(\widehat{S}):=\frac{1}{|\Omega|}\sum_{(u,i)\in\Omega}(\widehat{S}_{u,i}-X_{u,i})^{2}.
\]
If the probability of every entry in $X$ being revealed is the same (i.e., the matrix $P$ consists of only one unique nonzero value), then the loss $L_{\text{MSE}}(\widehat{S})$ is an unbiased estimate of the loss $L_{\text{full MSE}}(\widehat{S})$. However, this is no longer guaranteed to hold when entries are missing with different probabilities. To handle this more general setting, we can debias the loss $L_{\text{MSE}}$ by weighting each observation inversely by its propensity score, a technique referred to as inverse propensity scoring (IPS) or inverse probability weighting in causal inference \citep{thompson2012sampling,imbens2015causal,little2019statistical,schnabel2016recommendations}:
\begin{equation}
\label{eq:IPSweight}
L_{\text{IPS-MSE}}(\widehat{S}|P):=\frac{1}{mn}\sum_{(u,i)\in\Omega}\frac{(\widehat{S}_{u,i}-X_{u,i})^{2}}{P_{u,i}}.
\end{equation}
Assuming $P$ is known, the IPS loss $L_{\text{IPS-MSE}}(\widehat{S}|P)$ is an unbiased estimate for $L_{\text{full MSE}}(\widehat{S})$.%
\footnote{Note that $L_{\text{IPS-MSE}}(\widehat{S}|P)=\frac{1}{mn}\sum_{u=1}^m\sum_{i=1}^n \ind\{(u,i)\in\Omega\} \frac{(\widehat{S}_{u,i}-X_{u,i}^*)^{2}}{P_{u,i}}$. Taking the expectation with respect to which entries are revealed, $\mathbb{E}_{\Omega}[L_{\text{IPS-MSE}}(\widehat{S}|P)]=\frac{1}{mn}\sum_{u=1}^m\sum_{i=1}^n P_{u,i} \frac{(\widehat{S}_{u,i}-X_{u,i}^*)^{2}}{P_{u,i}}=L_{\text{full MSE}}(\widehat{S})$.}

Any matrix completion method that uses the naive MSE loss $L_{\text{MSE}}$ can then be modified to instead use the unbiased loss~$L_{\text{IPS-MSE}}$. For example, the standard approach of minimizing $L_{\text{MSE}}$ with nuclear norm regularization can be modified where we instead solve the following convex program:
\begin{equation}
\widehat{S} =
\underset{\Gamma\in\mathbb{R}^{m\times n}}{\text{arg min}}~ L_{\text{IPS-MSE}}(\Gamma|P) + \lambda \|\Gamma\|_*,
\label{eq:IPS-MSE-nuc-norm-opt}
\end{equation}
where $\lambda>0$ is a user-specified parameter, and $\|\cdot\|_*$ denotes the nuclear norm. Importantly, using the loss $L_{\text{IPS-MSE}}$ requires either knowing or having an estimate for the propensity score matrix $P$.

Instead of squared error, we could look at other kinds of error such as absolute error, in which case we would consider MAE instead of MSE. Also, instead of nuclear norm, other regularizers could be used in optimization problem~\eqref{eq:IPS-MSE-nuc-norm-opt}. Lastly, we remark that the inverse-propensity-scoring loss $L_{\text{IPS-MSE}}$ is not the only way to use propensity scores to weight. Another example is the Self-Normalized Inverse Propensity Scoring (SNIPS) estimator \citep{trotter1956conditional,swaminathan2015self}, which replaces the denominator term $mn$ in equation~\eqref{eq:IPSweight} by $\sum_{(u,i)\in\Omega} \frac1{P_{u,i}}$. This estimator tends to have lower variance than the IPS estimator but incurs a small bias \citep{hesterberg1995weighted}.

For ease of analysis, our theory focuses on debiasing with IPS. Algorithmically, for a given $P$, whether one uses IPS or SNIPS for estimating $S$ in optimization problem~\eqref{eq:IPS-MSE-nuc-norm-opt} does not matter since they differ by a multiplicative constant; tuning regularization parameter $\lambda$ would account for such constants. In experiments, for reporting test set errors, we use SNIPS since IPS can be quite sensitive to how many revealed entries are taken into consideration.






\textbf{Estimating the propensity score matrix.}
We can estimate $P$ based on the missingness mask matrix $M\in\{0,1\}^{m\times n}$, where $M_{u,i}=\ind\{X_{u,i}\ne\star\}$. Specifically, we use the nuclear-norm-constrained maximum likelihood estimator proposed by \citet{davenport20141} for 1-bit matrix completion, which we refer to as \textsc{1bitMC}. The basic idea of \textsc{1bitMC} is to model $P$ as the result of applying a user-specified link function $\sigma:\mathbb{R}\rightarrow[0,1]$ to each entry of a parameter matrix $A\in\mathbb{R}^{m\times n}$ so that $P_{u,i}=\sigma(A_{u,i})$; $\sigma$ can for instance be taken to be the standard logistic function $\sigma(x)=1/(1+e^{-x})$.
Then we estimate $A$ assuming that it satisfies nuclear norm and entry-wise max norm constraints, namely that
\[
A \in
\mathcal{F}_{\tau,\gamma}
\!:=\! \big\{ \Gamma\in\mathbb{R}^{m\times n}:\| \Gamma\|_{*}\le\tau\sqrt{mn},~ \| \Gamma\|_{\max} \le \gamma \big\},
\]
where $\tau>0$ and $\gamma>0$ are user-specified parameters. Then \textsc{1bitMC} is given as follows:

\begin{enumerate}[leftmargin=2em,topsep=0pt,partopsep=1ex,parsep=1ex]
\item Solve the constrained Bernoulli maximum likelihood problem:
\begin{equation}
\widehat{A}=\underset{\Gamma\in\mathcal{F}_{\tau,\gamma}}{\text{arg max}}\sum_{u=1}^{m}\sum_{i=1}^{n}[M_{u,i}\log \sigma(\Gamma_{u,i})+(1-M_{u,i})\log(1-\sigma(\Gamma_{u,i}))]. \label{eq:stage1-opt}
\end{equation}
For specific choices of $\sigma$ such as the standard logistic function, this optimization problem is convex and can, for instance, be solved via projected gradient descent.
\item Construct the matrix $\widehat{P}\in[0,1]^{m\times n}$, where $\widehat{P}_{u,i}:=\sigma(\widehat{A}_{u,i})$.
\end{enumerate}

\vspace{-1em}
\section{Theoretical Guarantee}
\label{sec:theory}
\vspace{-1em}

For $\widehat{P}$ computed via \textsc{1bitMC}, our theory bounds how close $\widehat{P}$ is to $P$ and also
how close the IPS loss $L_{\text{IPS-MSE}}(\widehat{S}|\widehat{P})$ is to the fully-observed MSE $L_{\text{full MSE}}(\widehat{S})$. We first state our assumptions on the propensity score matrix $P$ and the partially observed matrix $X$.
As introduced previously, \textsc{1bitMC} models $P$ via parameter matrix $A\in\mathbb{R}^{m\times n}$ and link function $\sigma:\mathbb{R}\rightarrow[0,1]$ such that $P_{u,i}=\sigma(A_{u,i})$. For ease of exposition, throughout this section, we take $\sigma$ to be the standard logistic function: $\sigma(x)={1/(1+e^{-x})}$.
Following \citet{davenport20141}, we assume that:
\begin{itemize}[leftmargin=2.5em,topsep=0pt,partopsep=1ex,parsep=0ex]
\item[\textbf{A1.}]$A$ has bounded nuclear norm: there exists a constant
$\theta\in(0,\infty)$ such that $\|A\|_{*}\le\theta\sqrt{mn}$.
\item[\textbf{A2.}]Entries of $A$ are bounded in absolute value: there exists a constant $\alpha\in(0,\infty)$ such that $\|A\|_{\max}:=\max_{u\in[m],i\in[n]}|A_{u,i}|\le\alpha$. In other words, $P_{u,i}\in[\sigma(-\alpha),\sigma(\alpha)]$ for all $u\in[m]$ and $i\in[n]$, where $\sigma$ is the standard logistic function.
\end{itemize}
As stated, Assumption A2 requires probabilities in $P$ to be bounded away from both 0 and 1. With small changes to \textsc{1bitMC} and our theoretical analysis, it is possible to allow for entries in $P$ to be~1, i.e., propensity scores should be bounded from~0 but not necessarily from~1. We defer discussing this setting to Appendix~\ref{sec:1bitMC-modified} as the changes are somewhat technical; the resulting theoretical guarantee is qualitatively similar to our guarantee for \textsc{1bitMC} below.

Assumptions A1 and A2 together are more general than assuming that $A$ has low rank and has entries bounded in absolute value. In particular, when Assumption A2 holds and $A$ has rank $r\in(0,\min\{m,n\}]$, then Assumption A1 holds with $\theta=\alpha\sqrt{r}$ (since $\|A\|_* \le \sqrt{r} \|A\|_F \le \sqrt{rmn} \|A\|_{\max} \le \alpha\sqrt{rmn}$, where $\|\cdot\|_F$ denotes the Frobenius norm). Note that a special case of $A$ being low rank is $A$ having clustering structure in rows, columns, or both. Thus, our theory also covers the case in which $P$ has row/column clustering with entries bounded away from 0. 

As for the partially observed matrix $X$, we assume that its values are bounded, regardless of which entries are revealed (so our assumption will be on $X^*$, the version of $X$ that is fully-observed):
\begin{itemize}[leftmargin=2.5em,topsep=0pt,partopsep=1ex,parsep=1ex]
\item[\textbf{A3.}]There exists a constant $\phi\in(0,\infty)$ such that $\|X^*\|_{\max}\le\phi$.
\end{itemize}
For example, in a recommendation systems context where $X$ is the ratings matrix, Assumption A3 holds if the ratings fall within a closed range of values (such as like/dislike where $X_{u,i}\in\{+1,-1\}$ and $\phi=1$, or a rating out of five stars where $X_{u,i}\in[1,5]$ and $\phi=5$).

For simplicity, we do not place assumptions on signal matrix $S$ or noise matrix $W$ aside from their sum $X^*$ having bounded entries. Different assumptions on $S$ and $W$ lead to different matrix completion algorithms. Many of these algorithms can be debiased using estimated propensity scores. We focus our theoretical analysis on this debiasing step and experimentally apply the debiasing to a variety of matrix completion algorithms. We remark that there are existing papers that discuss how to handle MNAR data when $S$ is low rank and $W$ consists of i.i.d.~zero-mean Gaussian (or sub-Gaussian) noise, a setup related to principal component analysis (e.g., \citet{sportisse2018imputation,sportisse2019estimation,zhu2019high}; a comparative study is provided by \citet{dray2015principal}).

\endgroup
\begingroup
\begingroup
\setlength\abovedisplayskip{2pt}
\setlength\belowdisplayskip{2pt}
\setlength\abovedisplayshortskip{1pt}
\setlength\belowdisplayshortskip{1pt}


Our main result is as follows. We defer the proof to Appendix~\ref{sec:pf-P-est-main-theorem}.
\begin{thm}\label{thm:P-est-main-theorem}
Under Assumptions A1--A3, suppose that we run algorithm \textsc{1bitMC} with user-specified parameters satisfying $\tau\ge\theta$ and $\gamma\ge\alpha$ to obtain the estimate $\widehat{P}$ of propensity score matrix $P$. Let $\widehat{S}\in\mathbb{R}^{m\times n}$ be any matrix satisfying $\|\widehat{S}\|_{\max}\le\psi$ for some $\psi\ge\phi$. Let $\delta\in(0,1)$. Then there exists a universal constant $C>0$ such that provided that $m+n\ge C$, with probability at least $1-\frac{C}{m+n}-\delta$ over randomness in which entries are revealed in $X$, we simultaneously have
\begin{align}
\frac{1}{mn}\sum_{u=1}^{m}\sum_{i=1}^{n}(\widehat{P}_{u,i}-P_{u,i})^{2}
& \le 4e\tau\Big(\frac{1}{\sqrt{m}}+\frac{1}{\sqrt{n}}\Big),
\label{eq:main-inequality1} \\[-1ex]
|L_{\text{IPS-MSE}}(\widehat{S}|\widehat{P})-L_{\text{full MSE}}(\widehat{S})|
& \le \frac{8\psi^2\sqrt{e\tau}}{\sigma(-\gamma)\sigma(-\alpha)}\Big(\frac{1}{m^{1/4}}+\frac{1}{n^{1/4}}\Big)
      + \frac{4\psi^{2}}{\sigma(-\alpha)}\sqrt{\frac{1}{2mn}\log\frac{2}{\delta}}.
\label{eq:main-inequality2}
\end{align}
\end{thm}
This theorem implies that under Assumptions A1--A3, with the number of rows and columns going to infinity, the IPS loss $L_{\text{IPS-MSE}}(\widehat{S}|\widehat{P})$ with $\widehat{P}$ computed using the \textsc{1bitMC} algorithm is a consistent estimator for the fully-observed MSE loss $L_{\text{full MSE}}(\widehat{S})$.

\endgroup
\begingroup

We remark that our result easily extends to using MAE instead of MSE. If we define
\[
L_{\text{full MAE}}(\widehat{S}):=\frac{1}{mn}\sum_{u=1}^m\sum_{i=1}^n |\widehat{S}_{u,i}-X^*_{u,i}|,\quad\;
L_{\text{IPS-MAE}}(\widehat{S}|P):=\frac{1}{mn}\sum_{(u,i)\in\Omega}\frac{|\widehat{S}_{u,i}-X_{u,i}|}{P_{u,i}},
\]
then the MAE version of Theorem \ref{thm:P-est-main-theorem} would replace \eqref{eq:main-inequality2} with
\[
|L_{\text{IPS-MAE}}(\widehat{S}|\widehat{P})-L_{\text{full MAE}}(\widehat{S})|
\le \frac{4\psi\sqrt{e\tau}}{\sigma(-\gamma)\sigma(-\alpha)}\Big(\frac{1}{m^{1/4}}+\frac{1}{n^{1/4}}\Big)
      + \frac{2\psi}{\sigma(-\alpha)}\sqrt{\frac{1}{2mn}\log\frac{2}{\delta}}.
\]
Equations~\eqref{eq:IPS-MSE-nuc-norm-opt} and \eqref{eq:stage1-opt} both correspond to convex programs that can be efficiently solved via proximal gradient methods \citep{parikh2014proximal}. Hence we can find a $\widehat{S}$ that minimizes $L_{\text{IPS-MSE}}(\widehat{S}|\widehat{P})$, and it is straightforward to show that when $m,n\to\infty$, this $\widehat{S}$ also minimizes $L_{\text{full MSE}}(\widehat{S})$ since $|L_{\text{IPS-MSE}}(\widehat{S}|\widehat{P})-L_{\text{full MSE}}(\widehat{S})| \to 0$.

\endgroup

\section{Experiments}

We now assess how well \textsc{1bitMC} debiases matrix completion algorithms on synthetic and real data.

\subsection{Synthetic Data}

\textbf{Data.} We create two synthetic datasets that are intentionally catered toward propensity scores being well-explained by naive Bayes and logistic regression. 1) \texttt{MovieLoverData}: the dataset comes from the Movie-Lovers toy example (Figure~1 in \citet{schnabel2016recommendations}, which is based on Table~1 of~\cite{steck2010training}), where we set parameter $p = 0.5$; 2) \texttt{UserItemData}: for the second dataset, the ``true'' rating matrix and propensity score matrix are generated by the following steps. We generate $U_1 \in [0,1]^{m \times 20}, V_1 \in [0,1]^{n \times 20}$ by sampling entries i.i.d.~from $\text{Uniform}[0,1]$, and then form the form $\widetilde{S}=U_1 V_1^{\top}$. We scale the values of $\widetilde{S}$ to be from 1 to 5 and round to the nearest integer to produce the true ratings matrix $S$. Next, we generate row and column feature vectors $U_2 \in \mathbb{R}^{m \times 20}, V_2 \in \mathbb{R}^{n \times 20}$ by sampling entries i.i.d.~from a normal distribution $\mathcal{N}(0,1/64)$. We further generate $w_1 \in [0,1]^{20 \times 1}, w_2 \in [0,1]^{20 \times 1}$ by sampling entries i.i.d.~from $\text{Uniform}[0,1]$. Then we form the propensity score matrix $P\in[0,1]^{m\times n}$ by setting $P_{u,i} = \sigma(U_2[u]w_1 + V_2[i]w_2)$, where~$\sigma$ is the standard logistic function, and $U_2[u]$ denotes the \mbox{$u$-th} row of $U_2$. For both datasets, we set $m=200, n = 300$. We also assume that i.i.d noise $\mathcal{N}(0,1)$ is added to each matrix entry of signal matrix~$S$ in producing the partially revealed matrix~$X$. All the ratings are clipped to $[1, 5]$ and rounded. By sampling based on $P$, we generate $\Omega$, the training set indices. 
The true ratings matrix~$S$ is used for testing.
We briefly explain why Assumptions A1--A3 hold for these two datasets in Appendix \ref{sec:more-experiments}.

\textbf{Algorithms comparison.}
We compare two types of algorithms for matrix completion. The first type does not account for entries being MNAR. This type of algorithm includes Probabilistic Matrix Factorization (\textsc{PMF}) \citep{mnih2008probabilistic}, Funk's \textsc{SVD} \citep{funk2006netflix}, \textsc{SVD++} \citep{koren2008factorization},  and \textsc{SoftImpute} \citep{mazumder2010spectral}. The second type accounts for MNAR entries and includes max-norm-constrained matrix completion (\textsc{MaxNorm}) \citep{cai2016matrix}, \textsc{expoMF} \citep{liang2016modeling}, and weighted-trace-norm-regularized matrix completion (\textsc{WTN}) \citep{srebro2010collaborative}. For all the algorithms above (except for \textsc{expoMF}), the ratings in the squared error loss can be debiased by the propensity scores (as shown in equation~\eqref{eq:IPSweight}), and the propensity scores can be estimated from logistic regression (\textsc{LR}) (which requires extra user or item feature data), naive Bayes (\textsc{NB}) (specifically equation (18) of \citet{schnabel2016recommendations}, which requires a small set of missing at random (MAR) ratings), and \textsc{1bitMC} \citep{davenport20141}. Hence we have a series of weighted-variants of the existing algorithms. For example, \textsc{1bitMC-PMF} means the  \textsc{PMF} method is used and the inverse propensity scores estimated from \textsc{1bitMC} is used as weights for debiasing.


\textbf{Metrics.} We use MSE and MAE to measure the estimation quality of the propensity scores. Similarly, we also use MSE and MAE to compare the estimated full rating matrix with the true rating matrix $S$ (denoted as full-MSE or full-MAE). We also report SNIPS-MSE (SNIPS-MAE); these are evaluated on test set entries (i.e., all matrix entries in these synthetic datasets) using the true $P$.


\textbf{Experiment setup.}
For all algorithms, we tune hyperparameters through 5-fold cross-validation using grid search. For the debiased methods (\textsc{LR-}$\star$, \textsc{NB-}$\star$, \textsc{1bitMC-}$\star$), we first estimate the propensity score matrix and then optimize the debiased loss.
We note that \texttt{MovieLoverData} does not contain user/item features. Thus, naive Bayes can be used to estimate $P$ for \texttt{MovieLoverData} and \texttt{UserItemData}, while logistic regression is only applicable for \texttt{UserItemData}. 
In using logistic regression to estimate propensity scores, we can use all user/item features, only user features, or only item features (denoted as \textsc{LR}, \textsc{LR-U}, \textsc{LR-I}, respectively). Per dataset, we generate $P$ and $S$ once before generating 10 samples of noisy revealed ratings $X$ based on $P$ and $S$. We apply all the algorithms stated above to these 10 experimental repeats. 

\textbf{Results.}
Before looking at the performance of matrix completion methods, we first inspect the accuracy of the estimated propensity scores. Since we know the true propensity score for the synthetic datasets, we can compare the true $P$ with the estimated $\widehat{P}$ directly, as presented in Table~\ref{tab:p-table}. In how we constructed the synthetic datasets, unsurprisingly estimating propensity scores using naive Bayes on \texttt{MovieLoverData} and logistic regression on \texttt{UserItemData} achieve the best performance. In both cases, \textsc{1bitMC} still achieves reasonably low errors in estimating the propensity score matrices.


\begin{table*}[h]
  \vspace{-0.5em}
  \centering
  \begin{tabular}{lllll}
    \toprule
    \multirow{2}{*}{Algorithm} & \multicolumn{2}{c}{\texttt{MovieLoverData}}  & \multicolumn{2}{c}{\texttt{UserItemData}}  \\
    \cmidrule(r){2-5}
    ~ & MSE & MAE  & MSE & MAE \\
    \midrule
    Naive Bayes  & {\bf 0.0346 $\pm$ 0.0002} & {\bf 0.1665 $\pm$ 0.0007 }&  0.0150 $\pm$ 0.0001& 0.0990 $\pm$ 0.0005   \\
   \textsc{lr} & N/A & N/A &{\bf 0.0002 $\pm$ 0.0001 }& {\bf 0.0105 $\pm$ 0.0017}   \\
    \textsc{lr-U}&  N/A & N/A  & 0.0070 $\pm$ 0.0000 &  0.0667 $\pm$ 0.0002   \\
    \textsc{lr-I} & N/A & N/A  & 0.0065 $\pm$ 0.0000 &  0.0639 $\pm$ 0.0001    \\
    \textsc{1bitMC} & {0.0520 $\pm$ 0.0003} & {0.1724 $\pm$ 0.0006} & {0.0119 $\pm$ 0.0000} &  {0.0881 $\pm$ 0.0002} \\
    \bottomrule
  \end{tabular}
  \smallskip
  \caption{Estimation accuracy of propensity score matrix (average $\pm$ standard deviation across 10 experimental repeats).}
  \label{tab:p-table}
\end{table*}


Now we compare the matrix completion methods directly and report the performance of different methods in Table~\ref{tab:mc}. Note that we only show the MSE-based results; the MAE-based results are presented in Appendix~\ref{sec:more-experiments}. 
The debiased variants generally perform as well as or better than their original unmodified counterparts.  \textsc{1bitMC-PMF} achieves the best accuracy on \texttt{MovieLoverData}, and both \textsc{SVD} and \textsc{1bitMC-SVD} perform the best on \texttt{UserItemData}. The debiasing using \textsc{1bitMC} can improve the performance of \textsc{PMF}, \textsc{SVD}, \textsc{MaxNorm} and \textsc{WTN} on \texttt{MovieLoverData}, and \textsc{PMF} is improved on \texttt{UserItemData}. In general, debiasing using \textsc{1bitMC} leads to higher matrix completion accuracy than debiasing using \textsc{LR} and \textsc{NB}.


\begin{table*}[t]
  \centering
  \begin{tabular}{lllll}
    \toprule
    \multirow{2}{*}{Algorithm} & \multicolumn{2}{c}{\texttt{MovieLoverData}}& \multicolumn{2}{c}{\texttt{UserItemData}}  \\
    \cmidrule(r){2-5}
    ~    &  MSE &  SNIPS-MSE & MSE & SNIPS-MSE  \\
    \midrule
\textsc{PMF} & 0.326  $\pm$  0.042 & 0.325  $\pm$  0.041 & 0.161  $\pm$  0.002 & 0.160  $\pm$  0.002 \\
\textsc{NB-PMF} & 0.363  $\pm$  0.013 & 0.363  $\pm$  0.012& 0.144  $\pm$  0.002 & 0.145  $\pm$  0.002 \\
\textsc{LR-PMF} &N/A & N/A& 0.159  $\pm$  0.002 & 0.164  $\pm$  0.003 \\
\textsc{1bitMC-PMF}  & {\bf 0.299  $\pm$  0.014} & {\bf 0.299  $\pm$  0.013} & 0.146  $\pm$  0.002 & 0.146  $\pm$  0.002\\
\arrayrulecolor{black!30}\midrule
\textsc{SVD}  & 1.359  $\pm$  0.033 & 1.360  $\pm$  0.034& {\bf 0.139  $\pm$  0.001} & {\bf 0.139  $\pm$  0.001} \\
\textsc{NB-SVD} & 0.866  $\pm$  0.028 & 0.866  $\pm$  0.027 & 0.147  $\pm$  0.001 & 0.147  $\pm$  0.002 \\
\textsc{LR-SVD} &N/A & N/A& 0.147  $\pm$  0.001 & 0.152  $\pm$  0.002 \\
\textsc{1bitMC-SVD} & 0.861  $\pm$  0.028 & 0.862  $\pm$  0.028& {\bf 0.139  $\pm$  0.001} & {\bf0.139  $\pm$  0.001} \\
\midrule
\textsc{SVD++} & 0.343  $\pm$  0.023 & 0.343  $\pm$  0.021 & 0.140  $\pm$  0.001 & 0.140  $\pm$  0.001 \\
\textsc{NB-SVD++} & 0.968  $\pm$  0.020 & 0.987  $\pm$  0.020 & 0.152  $\pm$  0.002 & 0.153  $\pm$  0.002 \\
\textsc{LR-SVD++} &N/A & N/A& 0.154  $\pm$  0.001 & 0.160  $\pm$  0.002 \\
\textsc{1bitMC-SVD++} & 0.345  $\pm$  0.023 & 0.345  $\pm$  0.021 & 0.140  $\pm$  0.001 & 0.140  $\pm$  0.001 \\
\midrule
\textsc{SoftImpute} & 0.374  $\pm$  0.009 & 0.374  $\pm$  0.008 & 0.579  $\pm$  0.002 & 0.556  $\pm$  0.003 \\
\textsc{NB-SoftImpute} & 0.495  $\pm$  0.010 & 0.495  $\pm$  0.009  & 0.599  $\pm$  0.004 & 0.588  $\pm$  0.004 \\
\textsc{LR-SoftImpute} &N/A & N/A& 0.602  $\pm$  0.003 & 0.581  $\pm$  0.004 \\
\textsc{1bitMC-SoftImpute} & 0.412  $\pm$  0.011 & 0.412  $\pm$  0.010 & 0.588  $\pm$  0.002 & 0.564  $\pm$  0.003 \\
\midrule
\textsc{MaxNorm} & 0.674  $\pm$  0.052 & 0.674  $\pm$  0.053 & 0.531  $\pm$  0.002 & 0.507  $\pm$  0.002 \\
\textsc{NB-MaxNorm} & 0.371  $\pm$  0.050 & 0.371  $\pm$  0.049 & 0.541  $\pm$  0.006 & 0.520  $\pm$  0.007 \\
\textsc{LR-MaxNorm} &N/A & N/A& 0.544  $\pm$  0.004 & 0.521  $\pm$  0.005 \\
\textsc{1bitMC-MaxNorm} & 0.396  $\pm$  0.036 & 0.395  $\pm$  0.035 & 0.542  $\pm$  0.003 & 0.519  $\pm$  0.004 \\
\midrule
\textsc{WTN} & 3.791  $\pm$  0.032 & 3.790  $\pm$  0.035  & 0.551  $\pm$  0.002 & 0.528  $\pm$  0.002 \\
\textsc{NB-WTN} & 3.262  $\pm$  0.093 & 3.262  $\pm$  0.094 & 0.557  $\pm$  0.002 & 0.535  $\pm$  0.002 \\
\textsc{LR-WTN} &N/A & N/A& 0.553  $\pm$  0.002 & 0.532  $\pm$  0.002 \\
\textsc{1bitMC-WTN} & 3.788  $\pm$  0.039 & 3.787  $\pm$  0.042 & 0.551  $\pm$  0.002 & 0.528  $\pm$  0.002 \\
\midrule
\textsc{ExpoMF} & 0.820  $\pm$  0.005 & 0.822  $\pm$  0.005 & 1.170  $\pm$  0.008 & 1.218  $\pm$  0.009 \\
\arrayrulecolor{black}\bottomrule
  \end{tabular}
  \smallskip
  \caption{MSE-based metrics of matrix completion methods on synthetic datasets (average $\pm$ standard deviation across 10 experimental repeats).}
  \label{tab:mc}
\end{table*}

\subsection{Real-World Data}
{\bf Data.} We consider two real-world datasets. 1) \texttt{Coat}: the dataset contains ratings from 290 users on 300 items \citep{schnabel2016recommendations}. The dataset contains both MNAR ratings as well as MAR ratings. Both user and item features are available for the dataset. 2) \texttt{MovieLens-100k}: the dataset contains 100k ratings from 943 users on 1,682 movies, and it does not contain any MAR ratings \citep{harper2016movielens}.

{\bf Experiments setup.} 
Since the \texttt{Coat} dataset contains both MAR and MNAR data, we are able to train the algorithms on the MNAR data and test on the MAR data. In this way, the MSE (MAE) on the testing set directly reflect the matrix completion accuracy. For \texttt{MovieLens-100k}, we split the data into 90/10 train/test sets $10$ times. For both datasets, we use 5-fold cross-validation to tune the hyperparameters through grid search. The SNIPS related measures are computed on test data based on the propensities estimated from \textsc{1bitMC-PMF} using training data.

{\bf Results.}  The performance of each algorithm is presented in Table~\ref{tab:real3}. We report the MSE-based results; MAE results are in Appendix~\ref{sec:more-experiments}. Algorithms \textsc{1bitMC-SoftImpute} and \textsc{1bitMC-PMF} perform the best on \texttt{Coat} based on MSE, and \textsc{1bitMC-SVD} outperforms the rest on \texttt{MovieLens-100k}. The debiasing approach does not improve the accuracy for \textsc{MaxNorm} and \textsc{WTN}.



\begin{table}[t]
  \centering
  \begin{tabular}{lllll}
    \toprule
    \multirow{2}{*}{Algorithm} & \multicolumn{2}{c}{\texttt{Coat}}&\multicolumn{2}{c}{\texttt{MovieLens-100k}}  \\
    \cmidrule(r){2-5}
    ~    & MSE & SNIPS-MSE & MSE & SNIPS-MSE\\
    \midrule
    \textsc{PMF} & 1.000 & {\bf1.051} &  0.896  $\pm$  0.013 & 0.902  $\pm$  0.013 \\
    \textsc{NB-PMF} & 1.034 & 1.117 & N/A & N/A\\
    \textsc{LR-PMF} & 1.025 & 1.107 & N/A & N/A\\
    \textsc{1bitMC-PMF} & 0.999 &  1.052 & 0.845  $\pm$  0.012 & 0.853  $\pm$  0.011 \\
    \arrayrulecolor{black!30}\midrule
    \textsc{SVD} & 1.203 & 1.270 & 0.862  $\pm$  0.013 & 0.872  $\pm$  0.012 \\
    \textsc{NB-SVD} & 1.246 & 1.346 & N/A & N/A\\
    \textsc{LR-SVD} & 1.234 & 1.334 & N/A & N/A\\
    \textsc{1bitMC-SVD} & 1.202 & 1.272  & {\bf0.821  $\pm$  0.011} & {\bf0.832  $\pm$  0.011} \\
    \arrayrulecolor{black!30}\midrule
    \textsc{SVD++} & 1.208 & 1.248 & 0.838  $\pm$  0.013 & 0.849  $\pm$  0.012 \\
    \textsc{NB-SVD++} & 1.488 & 1.608 & N/A & N/A\\
    \textsc{LR-SVD++} & 1.418 & 1.532 & N/A & N/A\\
    \textsc{1bitMC-SVD++} & 1.248 & 1.274 & 0.833  $\pm$  0.012 & 0.843  $\pm$  0.011 \\
    \arrayrulecolor{black!30}\midrule
    \textsc{SoftImpute} & 1.064 & 1.150 & 0.929  $\pm$  0.015 & 0.950  $\pm$  0.015 \\
    \textsc{NB-SoftImpute} & 1.052 & 1.138 & N/A & N/A\\
    \textsc{LR-SoftImpute} & 1.069 & 1.156 & N/A & N/A\\
    \textsc{1bitMC-SoftImpute} & {\bf 0.998} & 1.078 & 0.933  $\pm$  0.014 & 0.953  $\pm$  0.014 \\
    \arrayrulecolor{black!30}\midrule
    \textsc{MaxNorm} & 1.168 & 1.263 & 0.911  $\pm$  0.011 & 0.925  $\pm$  0.011 \\
    \textsc{NB-MaxNorm} & 1.460 & 1.578 & N/A & N/A\\
    \textsc{LR-MaxNorm} & 1.537 & 1.662 & N/A & N/A\\
    \textsc{1bitMC-MaxNorm} & 1.471 & 1.590 & 0.977  $\pm$  0.017 & 0.992  $\pm$  0.019 \\
    \arrayrulecolor{black!30}\midrule
    \textsc{WTN} & 1.396 & 1.509 &0.939  $\pm$  0.013 & 0.952  $\pm$  0.013 \\
    \textsc{NB-WTN} & 1.329 & 1.437 & N/A & N/A\\
    \textsc{LR-WTN} & 1.340 & 1.448 & N/A & N/A\\
    \textsc{1bitMC-WTN} & 1.396 & 1.509 & 0.934  $\pm$  0.013 & 0.946  $\pm$  0.013 \\
     \arrayrulecolor{black!30}\midrule
    \textsc{ExpoMF} & 2.602 & 2.813 & 2.461  $\pm$  0.077 & 2.558  $\pm$  0.083 \\
    \arrayrulecolor{black}\bottomrule
  \end{tabular}
  \vspace{0.5em}
  \caption{MSE-based metrics of matrix completion methods on \texttt{Coat} and \texttt{MovieLens-100k} (results for \texttt{MovieLens-100k} are the averages $\pm$ standard deviations across 10 experimental repeats).}
  \label{tab:real3}
  \vspace{-1em}
  \vspace{-.25em}
  \vspace{-.25em}
\end{table}

\vspace{-.25em}
\vspace{-.25em}
\vspace{-.25em}
\section{Conclusions}
\vspace{-.25em}
\vspace{-.25em}
\vspace{-.25em}

In this paper, we examined the effectiveness of debiasing matrix completion algorithms using missingness probabilities (propensity scores) estimated via another matrix completion algorithm: \textsc{1bitMC} by \citet{davenport20141}, which relies on low nuclear norm structure, and which we apply to a fully-revealed missingness mask matrix (so we are  doing matrix \textit{denoising} rather than completion).
%
Our numerical experiments indicate that debiasing using \textsc{1bitMC} can achieve downstream matrix completion accuracy at least as good as debiasing using logistic regression and naive Bayes baselines, despite \textsc{1bitMC} not using auxiliary information such as row/column feature vectors. Moreover, debiasing matrix completion algorithms with \textsc{1bitMC} can boost accuracy, in some cases achieving the best or nearly the best performance across all algorithms we tested. These experimental findings suggest that a low nuclear norm assumption on missingness patterns is reasonable.



In terms of theoretical analysis, we have not addressed the full generality of MNAR data in matrix completion. For example, we still assume that each entry is revealed independent of other entries. In reality, one matrix entry being revealed could increase (or decrease) the chance of another entry being revealed. As another loose end, our theory breaks down when a missingness probability is exactly 0. For example, consider when the matrix to be completed corresponds to feature vectors collected from patients. A clinical measurement that only makes sense for women will have 0 probability of being revealed for men. In such scenarios, imputing such a missing value does not actually make sense. These are two open problems among many for robustly handling MNAR data with guarantees.



\clearpage

\bibliographystyle{plainnat}
\bibliography{nuclear_mnar}
\appendix
\newpage

\pagebreak
\begin{center}
\textbf{\Large Supplemental Material}
\end{center}

\section{Topic Modeling Structure}
\label{sec:topic-modeling}

Not only do the \texttt{Coat} and \texttt{MovieLens-100k} datasets have block structure, they can also be explained using topic modeling structure, which is again a low rank model for the propensity score matrix~$P$ (and is thus a special case of the general low nuclear norm structure we assume provided that~$P$ be bounded away from 0). We build on our biclustering data exploration example from Figure~\ref{fig:block-structure}. Specifically, in the \texttt{MovieLens-100k} dataset, each user can be thought of as a distribution over movie ``topics'', and each movie topic corresponds to a distribution over movie genres, as shown in Figure~\ref{fig:pr33}. Meanwhile, in the \texttt{Coat} dataset, each item can be represented as a distribution over user topics, and each user topic can be thought of as a distribution over user features, as shown in Figure~\ref{fig:pr34}.


\begin{figure}[h]
  \centering
  \includegraphics[width = 0.75\linewidth]{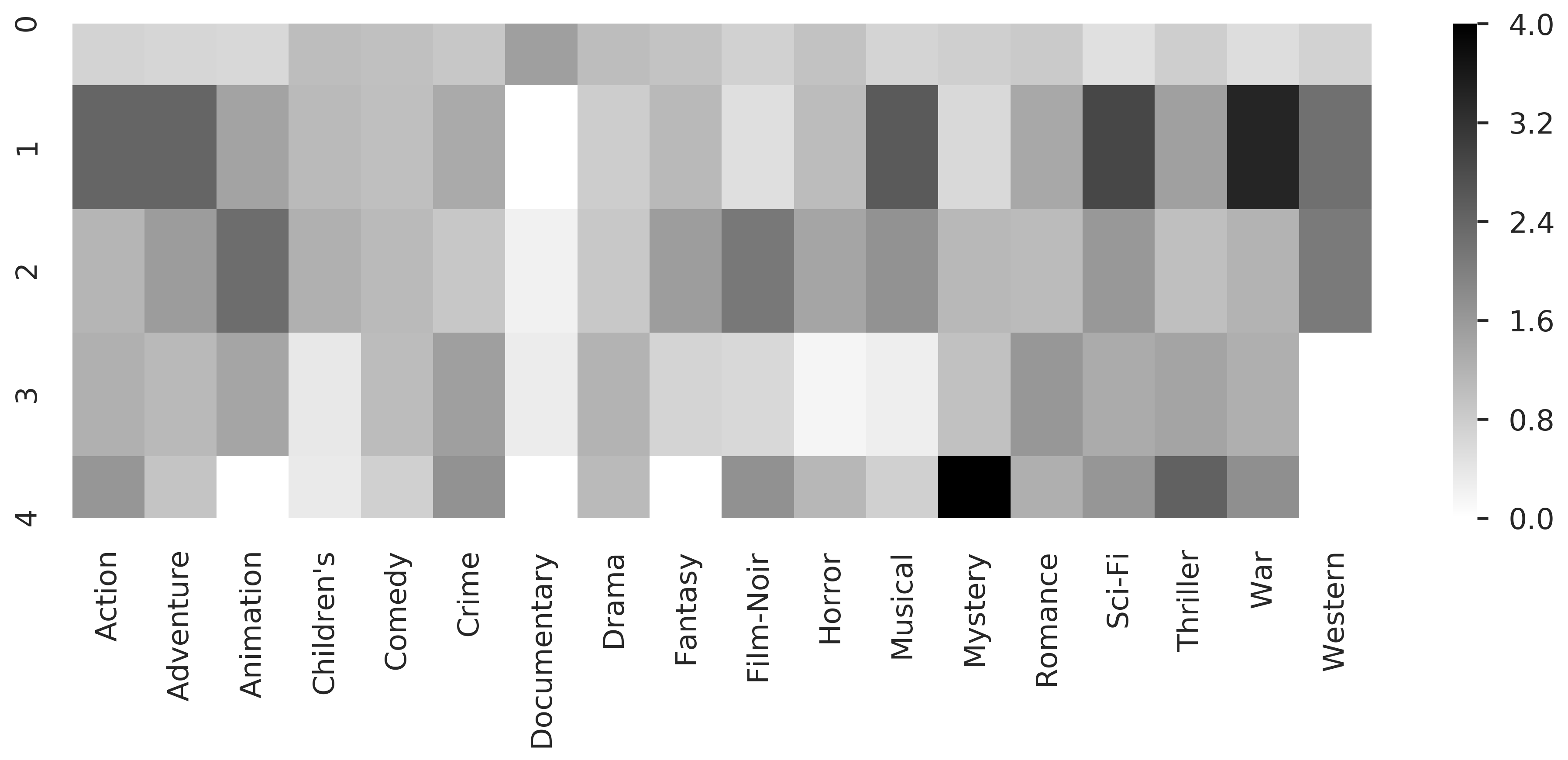}
  \caption{Average item features for each movie topic in  \texttt{MovieLens-100k}.}
  \label{fig:pr33}
\end{figure}

\begin{figure}[h]
  \centering
  \includegraphics[width = 0.75\linewidth]{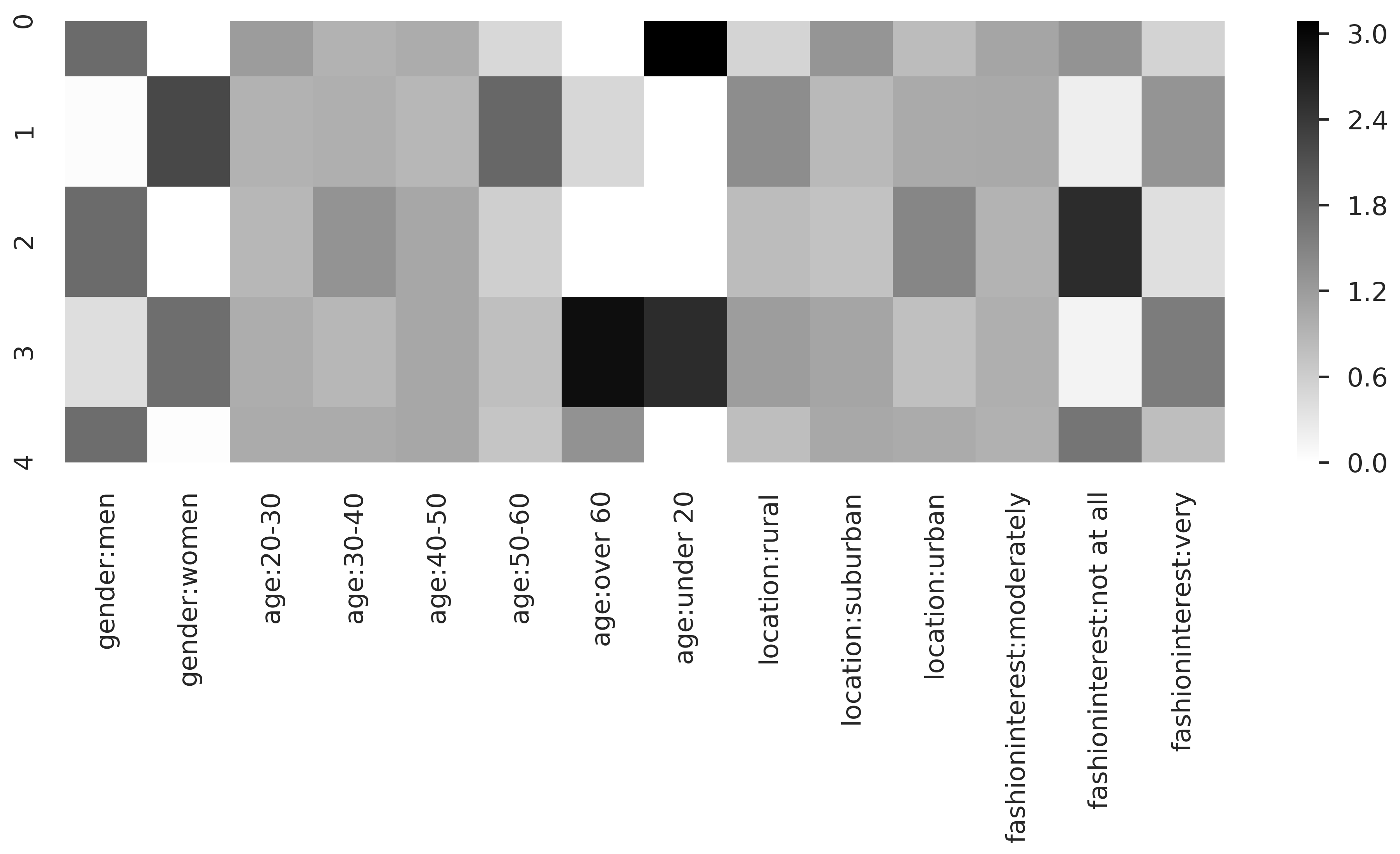}
  \caption{Average user features for each user topic in  \texttt{Coat}.}
  \label{fig:pr34}
\end{figure}

Topic models reveal co-occurrence information.
For example, movie topic~1 in Figure~\ref{fig:pr33} corresponds to users who tend to reveal ratings for action, adventure, musical, sci-fi, and war movies, but not documentaries. Movie topic~4 corresponds to users who tend to reveal ratings for mysteries and thrillers. Users can be associated with multiple topics to varying degrees. Ratings for documentaries tend to mostly be revealed by users associated with movie topic~0. We can also find such patterns in the \texttt{Coat} dataset, where we look at topics over rows instead of columns as an illustration. We can see what sorts of user features tend to be co-occur. For example, user topic 2 consists of men uninterested in fashion. Coats can be associated with different user topics to varying degrees.

Note that the biclustering uses only the missingness information. We are interpreting the biclustering results here with the help of user/item features to identify co-occurrence relationships.

\section{Proof of Theorem \ref{thm:P-est-main-theorem}}
\label{sec:pf-P-est-main-theorem}

Denote the objective function (which is a log likelihood) of optimization problem \eqref{eq:stage1-opt} as
\[
L_{M}(\Gamma):=\sum_{u=1}^{m}\sum_{i=1}^{n}[M_{u,i}\log \sigma(\Gamma_{u,i})+(1-M_{u,i})\log(1-\sigma(\Gamma_{u,i}))].
\]
We specialize Theorem 2 of \citet{davenport20141} to the setting in
which all entries of the matrix are observed. Their proof makes use
of an offset version of the log likelihood function:
\begin{align*}
\overline{L}_{M}(\Gamma) & :=L_{M}(\Gamma)-L_{M}(\bm{0}) 
 =\sum_{u=1}^{m}\sum_{i=1}^{n}\Big[M_{u,i}\log\frac{\sigma(\Gamma_{u,i})}{\sigma(0)}+(1-M_{u,i})\log\frac{1-\sigma(\Gamma_{u,i})}{1-\sigma(0)}\Big].
\end{align*}
We let $\mathcal{F}_{\tau,\gamma}$ denote the feasible set of optimization problem \eqref{eq:stage1-opt}, i.e., 
\[
\mathcal{F}_{\tau,\gamma}:=\big\{ \Gamma\in\mathbb{R}^{m\times n}\;:\;\|\Gamma\|_{*}\le\tau\sqrt{mn},\; \|\Gamma\|_{\max}\le\gamma \big\}.
\]
Moreover, we actually let the function $\sigma$ be a bit more general: $\sigma$ must be differentiable, and the following quantity must exist and be finite:
\begin{equation}
L_{\gamma}:=\sup_{x\in[-\gamma,\gamma]}\frac{|\sigma'(x)|}{\sigma(x)(1-\sigma(x))}. 
\label{eq:weird-L}
\end{equation}
We prove the following more general theorem.
\begin{thm}
\label{thm:P-est-helper}
Under Assumptions A1--A3, suppose that we run algorithm \textsc{1bitMC} with user-specified parameters satisfying $\tau\ge\theta$ and $\gamma\ge\alpha$ to obtain the estimate $\widehat{P}$ of propensity score matrix $P$. Let $\widehat{S}\in\mathbb{R}^{m\times n}$ be any matrix satisfying $\|\widehat{S}\|_{\max}\le\psi$ for some $\psi\ge\phi$. Let $\delta\in(0,1)$. Then there exists a universal constant $C>0$ such that provided that $m+n\ge C$, with probability at least $1-\frac{C}{m+n}-\delta$ over randomness in which entries are revealed in $X$, we simultaneously have
\begin{align}
\frac{1}{mn}\sum_{u=1}^{m}\sum_{i=1}^{n}(\widehat{P}_{u,i}-P_{u,i})^{2}
& \le4eL_{\gamma}\tau\Big(\frac{1}{\sqrt{m}}+\frac{1}{\sqrt{n}}\Big),
\label{eq:main-inequality1-more-general} \\
|L_{\text{IPS-MSE}}(\widehat{S}|\widehat{P})-L_{\text{full MSE}}(\widehat{S})|
& \le \frac{8\psi^2\sqrt{eL_{\gamma}\tau}}{\sigma(-\gamma)\sigma(-\alpha)}\Big(\frac{1}{m^{1/4}}+\frac{1}{n^{1/4}}\Big)
      + \frac{4\psi^{2}}{\sigma(-\alpha)}\sqrt{\frac{1}{2mn}\log\frac{2}{\delta}}.
\label{eq:main-inequality2-more-general}
\end{align}
\end{thm}
We recover Theorem \ref{thm:P-est-main-theorem} by noting that for $\sigma$ chosen to be the
standard logistic function, we have $L_{\gamma}=1$ for all $\gamma>0$.

\subsection*{Proof of Theorem \ref{thm:P-est-helper}}

Under Assumption A1 and since $\tau\ge\theta$, note that $A\in\mathcal{F}_{\tau,\gamma}$.
By optimality of $\widehat{A}$ for optimization problem \eqref{eq:stage1-opt},
we have $L_{M}(\widehat{A})\ge L_{M}(A)$, which can written as
\begin{align*}
0 & \le L_{M}(\widehat{A})-L_{M}(A)\\
 & =\overline{L}_{M}(\widehat{A})-\overline{L}_{M}(A)\\
 & =(\overline{L}_{M}(\widehat{A})-\mathbb{E}_{M}[\overline{L}_{M}(\widehat{A})])-(\overline{L}_{M}(A)-\mathbb{E}_{M}[\overline{L}_{M}(A)])+\mathbb{E}_{M}[\overline{L}_{M}(\widehat{A})-\overline{L}_{M}(A)].
\end{align*}
Since matrices $\widehat{A}$ and $A$ are both in the set $\mathcal{F}_{\tau,\gamma}$,
the first two terms on the right-hand side can each be upper-bounded
by $\sup_{\Gamma\in\mathcal{F}_{\tau,\gamma}}|\overline{L}_{M}(\Gamma)-\mathbb{E}_{M}[\overline{L}_{M}(\Gamma)]|$.
Meanwhile, the third term on the right-hand side can be rewritten
as
\begin{align*}
 \mathbb{E}_{M}[\overline{L}_{M}(\widehat{A})-\overline{L}_{M}(A)]
 & =\sum_{u=1}^{m}\sum_{i=1}^{n}\mathbb{E}_{M_{u,i}}\Big[M_{u,i}\log\frac{\sigma(\widehat{A}_{u,i})}{\sigma(A_{u,i})}+(1-M_{u,i})\log\frac{1-\sigma(\widehat{A}_{u,i})}{1-\sigma(A_{u,i})}\Big]\\
 & =\sum_{u=1}^{m}\sum_{i=1}^{n}\Big[P_{u,i}\log\frac{\sigma(\widehat{A}_{u,i})}{\sigma(A_{u,i})}+(1-P_{u,i})\log\frac{1-\sigma(\widehat{A}_{u,i})}{1-\sigma(A_{u,i})}\Big]\\
 & =\sum_{u=1}^{m}\sum_{i=1}^{n}\Big[P_{u,i}\log\frac{\widehat{P}_{u,i}}{P_{u,i}}+(1-P_{u,i})\log\frac{1-\widehat{P}_{u,i}}{1-P_{u,i}}\Big]\\
 & =-\sum_{u=1}^{m}\sum_{i=1}^{n}\Big[P_{u,i}\log\frac{P_{u,i}}{\widehat{P}_{u,i}}+(1-P_{u,i})\log\frac{1-P_{u,i}}{1-\widehat{P}_{u,i}}\Big]\\
 & =-\sum_{u=1}^{m}\sum_{i=1}^{n}D\big(\text{Ber}(P_{u,i})\|\text{Ber}(\widehat{P}_{u,i})\big).
\end{align*}
Putting together the pieces, we have
\begin{align}
0 & \le(\overline{L}_{M}(\widehat{A})-\mathbb{E}_{M}[\overline{L}_{M}(\widehat{A})])-(\overline{L}_{M}(A)-\mathbb{E}_{M}[\overline{L}_{M}(A)])+\mathbb{E}_{M}[\overline{L}_{M}(\widehat{A})-\overline{L}_{M}(A)]\nonumber \\
 & \le2\sup_{\Gamma\in\mathcal{F}_{\tau,\gamma}}|\overline{L}_{M}(\Gamma)-\mathbb{E}_{M}[\overline{L}_{M}(\Gamma)]|+\mathbb{E}_{M}[\overline{L}_{M}(\widehat{A})-\overline{L}_{M}(A)]\nonumber \\
 & =2\sup_{\Gamma\in\mathcal{F}_{\tau,\gamma}}|\overline{L}_{M}(\Gamma)-\mathbb{E}_{M}[\overline{L}_{M}(\Gamma)]|-\sum_{u=1}^{m}\sum_{i=1}^{n}D\big(\text{Ber}(P_{u,i})\|\text{Ber}(\widehat{P}_{u,i})\big).\label{eq:P-est-helper1}
\end{align}
By Pinsker's inequality,
\begin{align*}
D\big(\text{Ber}(P_{u,i})\|\text{Ber}(\widehat{P}_{u,i})\big) & \ge2\|\text{Ber}(P_{u,i})-\text{Ber}(\widehat{P}_{u,i})\|_{\text{TV}}^{2}\\
 & =2\Big[\frac{1}{2}(|P_{u,i}-\widehat{P}_{u,i}|+|(1-P_{u,i})-(1-\widehat{P}_{u,i})|)\Big]^{2}\\
 & =2(\widehat{P}_{u,i}-P_{u,i})^{2}.
\end{align*}
Therefore,
\begin{align}
\sum_{u=1}^{m}\sum_{i=1}^{n}D\big(\text{Ber}(P_{u,i})\|\text{Ber}(\widehat{P}_{u,i})\big) & \ge2\sum_{u=1}^{m}\sum_{i=1}^{n}(\widehat{P}_{u,i}-P_{u,i})^{2}.\label{eq:P-est-helper2}
\end{align}
Combining inequalities \eqref{eq:P-est-helper1} and \eqref{eq:P-est-helper2},
we get
\[
\sum_{u=1}^{m}\sum_{i=1}^{n}(\widehat{P}_{u,i}-P_{u,i})^{2}\le\sup_{\Gamma\in\mathcal{F}_{\tau,\gamma}}|\overline{L}_{M}(\Gamma)-\mathbb{E}_{M}[\overline{L}_{M}(\Gamma)]|.
\]
The next lemma upper-bounds $\sup_{\Gamma\in\mathcal{F}_{\tau,\gamma}}|\overline{L}_{M}(\Gamma)-\mathbb{E}_{M}[\overline{L}_{M}(\Gamma)]|$.
\begin{lem}
\label{lem:P-est-main-lemma}For the above setup, if $m+n\ge3$, then
there exists a universal constant $C>0$ such that
\begin{equation}
\mathbb{P}\Big(\sup_{\Gamma\in\mathcal{F}_{\tau,\gamma}}|\overline{L}_{M}(\Gamma)-\mathbb{E}_{M}[\overline{L}_{M}(\Gamma)]|\ge4eL_{\gamma}\tau\sqrt{mn}(\sqrt{m}+\sqrt{n})\Big)\le\frac{C}{m+n}.
\label{eq:bad-event1-bound}
\end{equation}
\end{lem}
Once this lemma is established, the theorem's first main inequality \eqref{eq:main-inequality1-more-general} readily follows since with probability at least $1-\frac{C}{m+n}$ (for which we clearly want $m+n\ge C$),
\begin{align*}
\frac{1}{mn}\sum_{u=1}^{m}\sum_{i=1}^{n}(\widehat{P}_{u,i}-P_{u,i})^{2} & \le\frac{1}{mn}[4eL_{\gamma}\tau\sqrt{mn}(\sqrt{m}+\sqrt{n})] 
 =4eL_{\gamma}\tau\Big(\frac{1}{\sqrt{m}}+\frac{1}{\sqrt{n}}\Big),
\end{align*}
which establishes inequality \eqref{eq:main-inequality1}.
Note that Lemma \ref{lem:P-est-main-lemma} asks that $m+n\ge3$. Since $C=8\cdot2^{1/4}\cdot e^2=70.2969\ldots$, asking that $m+n\ge C$ implies that $m+n\ge3$.

We now derive the theorem's second main inequality \eqref{eq:main-inequality2-more-general}, which is a consequence of the first main inequality \eqref{eq:main-inequality1-more-general}. By the triangle inequality,
\begin{align}
&|L_{\text{IPS-MSE}}(\widehat{S}|\widehat{P})-L_{\text{full MSE}}(\widehat{S})| \nonumber \\
&\quad \le
|L_{\text{IPS-MSE}}(\widehat{S}|\widehat{P})-L_{\text{IPS-MSE}}(\widehat{S}|P)|
+ |L_{\text{IPS-MSE}}(\widehat{S}|P)-L_{\text{full MSE}}(\widehat{S})|.
\label{eq:main-triangle}
\end{align}
We can readily bound the first RHS term as follows:
\begingroup
\allowdisplaybreaks
\begin{align*}
|L_{\text{IPS-MSE}}(\widehat{S}|\widehat{P})-L_{\text{IPS-MSE}}(\widehat{S}|P)| & =\bigg|\frac{1}{mn}\sum_{(u,i)\in\Omega}\Big[\frac{(\widehat{S}_{u,i}-X_{u,i})^{2}}{\widehat{P}_{u,i}}-\frac{(\widehat{S}_{u,i}-X_{u,i})^{2}}{P_{u,i}}\Big]\bigg|\\
 & =\bigg|\frac{1}{mn}\sum_{(u,i)\in\Omega}\frac{P_{u,i}-\widehat{P}_{u,i}}{\widehat{P}_{u,i}P_{u,i}}(\widehat{S}_{u,i}-X_{u,i})^{2}\bigg|\\
 \text{worst case error }|\widehat{S}_{u,i}-X_{u,i}|\le2\psi\quad
 & \le\frac{4\psi^{2}}{mn}\sum_{(u,i)\in\Omega}\frac{|P_{u,i}-\widehat{P}_{u,i}|}{\widehat{P}_{u,i}P_{u,i}} \\
 \text{note that }\widehat{P}_{u,i}\ge \sigma(-\gamma)\text{ and }P_{u,i}\ge \sigma(-\alpha)\quad
 & \le\frac{4\psi^{2}}{\sigma(-\gamma)\sigma(-\alpha)mn}\sum_{(u,i)\in\Omega}|P_{u,i}-\widehat{P}_{u,i}| \\
 \text{basic inequality relating $\ell_1$ and $\ell_2$ norms}\quad
 & \le\frac{4\psi^{2}}{\sigma(-\gamma)\sigma(-\alpha)mn}\sqrt{|\Omega|}\sqrt{\sum_{(u,i)\in\Omega}(\widehat{P}_{u,i}-P_{u,i})^{2}}\\
 & =\frac{4\psi^{2}}{\sigma(-\gamma)\sigma(-\alpha)}\sqrt{\frac{|\Omega|}{mn}}\sqrt{\frac{1}{mn}\sum_{(u,i)\in\Omega}(\widehat{P}_{u,i}-P_{u,i})^{2}} \\
 \text{the theorem's main inequality \eqref{eq:main-inequality1-more-general}}\quad
 & \le\frac{4\psi^{2}}{\sigma(-\gamma)\sigma(-\alpha)}\sqrt{\frac{|\Omega|}{mn}}\sqrt{4eL_{\gamma}\tau\Big(\frac{1}{\sqrt{m}}+\frac{1}{\sqrt{n}}\Big)} \\
 \text{fraction of observed entries ${\textstyle \frac{|\Omega|}{mn}}$ is at most 1}\quad
 & \le\frac{4\psi^{2}}{\sigma(-\gamma)\sigma(-\alpha)}\sqrt{4eL_{\gamma}\tau\Big(\frac{1}{\sqrt{m}}+\frac{1}{\sqrt{n}}\Big)}\\
 & = \frac{8\psi^{2}\sqrt{eL_{\gamma}\tau}}{\sigma(-\gamma)\sigma(-\alpha)}\sqrt{\frac{1}{\sqrt{m}}+\frac{1}{\sqrt{n}}} \\
 & \le\frac{8\psi^{2}\sqrt{eL_{\gamma}\tau}}{\sigma(-\gamma)\sigma(-\alpha)}\Big(\frac{1}{m^{1/4}}+\frac{1}{n^{1/4}}\Big),
\end{align*}
\endgroup
where the last step uses the fact that $(a+b)^p \le a^p + b^p$ for all $p\in[0,1]$ and $a,b\in\mathbb{R}_+$.

The second RHS term in inequality \eqref{eq:main-triangle} can be bounded with Hoeffding's inequality. First, recall that
\begin{align*}
L_{\text{IPS-MSE}}(\widehat{S}|P)=\frac{1}{mn}\sum_{u=1}^m\sum_{i=1}^n \ind\{(u,i)\in\Omega\} \frac{(\widehat{S}_{u,i}-X_{u,i}^*)^{2}}{P_{u,i}},
\end{align*}
which is an average of $mn$ random variables (here, the only randomness we are considering is in which entries are revealed $\Omega$). One can check that $\mathbb{E}_{\Omega}[L_{\text{IPS-MSE}}(\widehat{S}|P)]=L_{\text{full MSE}}(\widehat{S})$. Note that $\frac{(\widehat{S}_{u,i}-X_{u,i}^*)^{2}}{P_{u,i}}\le\frac{(2\psi)^2}{\sigma(-\alpha)}=\frac{4\psi^2}{\sigma(-\alpha)}$. Thus, each of the terms in the double summation above is bounded in the interval $[0,\frac{4\psi^2}{\sigma(-\alpha)}]$, so by Hoeffding's inequality,
\begin{equation}
\mathbb{P}\bigg(\, |L_{\text{IPS-MSE}}(\widehat{S}|P) - L_{\text{full MSE}}(\widehat{S})| \ge \frac{4\psi^{2}}{\sigma(-\alpha)}\sqrt{\frac{1}{2mn}\log\frac{2}{\delta}}\,\,\bigg)
\le \delta.\label{eq:bias-hoeffding-bound}
\end{equation}
When this bad event does not happen, then the second RHS term of triangle inequality \eqref{eq:main-triangle} is at most $\frac{4\psi^{2}}{\sigma(-\alpha)}\sqrt{\frac{1}{2mn}\log\frac{2}{\delta}}$, so putting together the pieces, we get the theorem's second main inequality~\eqref{eq:main-inequality2-more-general}.
By a union bound, the bad events corresponding to bounds \eqref{eq:bad-event1-bound} and \eqref{eq:bias-hoeffding-bound} both don't happen with probability at least $1-\frac{C}{m+n}-\delta$.

\begin{proof}[Proof of Lemma \ref{lem:P-est-main-lemma}]
By Markov's inequality, for any $h>0$ and $z>0$, we have
\begin{align}
 \mathbb{P}\Big(\sup_{\Gamma\in\mathcal{F}_{\tau,\gamma}}|\overline{L}_{M}(\Gamma)-\mathbb{E}_{M}[\overline{L}_{M}(\Gamma)]|\ge z\Big)
 & =\mathbb{P}\Big(\sup_{\Gamma\in\mathcal{F}_{\tau,\gamma}}|\overline{L}_{M}(\Gamma)-\mathbb{E}_{M}[\overline{L}_{M}(\Gamma)]|^{h}\ge z^{h}\Big)\nonumber \\
 & \quad \le\frac{\mathbb{E}_{M}\big[\sup_{\Gamma\in\mathcal{F}_{\tau,\gamma}}|\overline{L}_{M}(\Gamma)-\mathbb{E}_{M}[\overline{L}_{M}(\Gamma)]|^{h}\big]}{z^{h}}.\label{eq:P-est-markov}
\end{align}
We will be setting $h=\log(m+n)$ (which is greater than 1 under the
assumption that $m+n\ge3$) and
$
z=4eL_{\gamma}\tau\sqrt{mn}(\sqrt{m}+\sqrt{n}).
$

We next upper-bound the numerator term $\mathbb{E}_{M}\big[\sup_{\Gamma\in\mathcal{F}_{\tau,\gamma}}|\overline{L}_{M}(\Gamma)-\mathbb{E}_{M}[\overline{L}_{M}(\Gamma)]|^{h}\big]$.
To do this, we apply a standard symmetrization argument. Let matrix
$M'\in\mathbb{R}^{m\times n}$ be independently sampled such that
$M'$ and $M$ have the same distribution. Then using Jensen's inequality,
\begingroup
\allowdisplaybreaks
\begin{align}
 \mathbb{E}_{M}\Big[\sup_{\Gamma\in\mathcal{F}_{\tau,\gamma}}|\overline{L}_{M}(\Gamma)-\mathbb{E}_{M}[\overline{L}_{M}(\Gamma)]|^{h}\Big]
 & =\mathbb{E}_{M}\Big[\sup_{\Gamma\in\mathcal{F}_{\tau,\gamma}}|\overline{L}_{M}(\Gamma)-\mathbb{E}_{M}[\overline{L}_{M}(\Gamma)]|^{h}\Big]\nonumber \\
 & =\mathbb{E}_{M}\Big[\sup_{\Gamma\in\mathcal{F}_{\tau,\gamma}}|\overline{L}_{M}(\Gamma)-\mathbb{E}_{M'}[\overline{L}_{M'}(\Gamma)]|^{h}\Big]\nonumber \\
 & =\mathbb{E}_{M}\Big[\sup_{\Gamma\in\mathcal{F}_{\tau,\gamma}}|\mathbb{E}_{M'}[\overline{L}_{M}(\Gamma)-\overline{L}_{M'}(\Gamma)]|^{h}\Big]\nonumber \\
 & \le\mathbb{E}_{M}\bigg[\mathbb{E}_{M'}\Big[\sup_{\Gamma\in\mathcal{F}_{\tau,\gamma}}|\overline{L}_{M}(\Gamma)-\overline{L}_{M'}(\Gamma)|^{h}\Big]\bigg]\nonumber \\
 & =\mathbb{E}_{M,M'}\Big[\sup_{\Gamma\in\mathcal{F}_{\tau,\gamma}}|\overline{L}_{M}(\Gamma)-\overline{L}_{M'}(\Gamma)|^{h}\Big].\label{eq:P-est-helper3}
\end{align}
\endgroup
In applying Jensen's inequality, we are using the fact that as a function
of $M'$, the function $\sup_{\Gamma\in\mathcal{F}_{\tau,\gamma}}|\overline{L}_{M}(\Gamma)-\overline{L}_{M'}(\Gamma)|^{h}$
(for $h\ge1)$ is the pointwise supremum of convex functions, so it
is still convex.

Next, we examine the random variable $\overline{L}_{M}(\Gamma)-\overline{L}_{M'}(\Gamma)$.
We shall be introducing independently sampled Rademacher random variables
$\xi_{u,i}\in\{\pm1\}$ for $u\in[m]$ and $i\in[n]$. Note that
\begin{align*}
 \overline{L}_{M}(\Gamma)-\overline{L}_{M'}(\Gamma)
 & =\sum_{u=1}^{m}\sum_{i=1}^{n}\Big[M_{u,i}\log\frac{\sigma(\Gamma_{u,i})}{\sigma(0)}+(1-M_{u,i})\log\frac{1-\sigma(\Gamma_{u,i})}{1-\sigma(0)}\Big]\\
 & \quad-\sum_{u=1}^{m}\sum_{i=1}^{n}\Big[M_{u,i}'\log\frac{\sigma(\Gamma_{u,i})}{\sigma(0)}+(1-M_{u,i}')\log\frac{1-\sigma(\Gamma_{u,i})}{1-\sigma(0)}\Big]\\
 & =\sum_{u=1}^{m}\sum_{i=1}^{n}\Big[M_{u,i}\log\frac{\sigma(\Gamma_{u,i})}{\sigma(0)}+(1-M_{u,i})\log\frac{1-\sigma(\Gamma_{u,i})}{1-\sigma(0)}\\
 & \qquad\qquad\;\,-M_{u,i}'\log\frac{\sigma(\Gamma_{u,i})}{\sigma(0)}-(1-M_{u,i}')\log\frac{1-\sigma(\Gamma_{u,i})}{1-\sigma(0)}\Big]
\end{align*}
has the same distribution as the random variable
\begin{align*}
 \overline{L}_{M}(\Gamma)-\overline{L}_{M'}(\Gamma)
 & =\sum_{u=1}^{m}\sum_{i=1}^{n}\xi_{u,i}\Big[M_{u,i}\log\frac{\sigma(\Gamma_{u,i})}{\sigma(0)}+(1-M_{u,i})\log\frac{1-\sigma(\Gamma_{u,i})}{1-\sigma(0)}\\
 & \qquad\qquad\;\,\quad\quad-M_{u,i}'\log\frac{\sigma(\Gamma_{u,i})}{\sigma(0)}-(1-M_{u,i}')\log\frac{1-\sigma(\Gamma_{u,i})}{1-\sigma(0)}\Big].
\end{align*}
Then, using the fact that $|a+b|^{p}\le2^{p-1}(|a|^{p}+|b|^{p})$
for $p\ge1$ and $a,b\in\mathbb{R}$,
\begin{align}
 & \mathbb{E}_{M,M'}\Big[\sup_{\Gamma\in\mathcal{F}_{\tau,\gamma}}|\overline{L}_{M}(\Gamma)-\overline{L}_{M'}(\Gamma)|^{h}\Big]\nonumber \\
 & =\mathbb{E}_{M,M'}\Bigg[\sup_{\Gamma\in\mathcal{F}_{\tau,\gamma}}\bigg|\sum_{u=1}^{m}\sum_{i=1}^{n}\xi_{u,i}\Big[M_{u,i}\log\frac{\sigma(\Gamma_{u,i})}{\sigma(0)}+(1-M_{u,i})\log\frac{1-\sigma(\Gamma_{u,i})}{1-\sigma(0)}\nonumber \\
 & \qquad\qquad\;\,\quad\quad\qquad\qquad\qquad-M_{u,i}'\log\frac{\sigma(\Gamma_{u,i})}{\sigma(0)}-(1-M_{u,i}')\log\frac{1-\sigma(\Gamma_{u,i})}{1-\sigma(0)}\Big]\bigg|^{h}\Bigg]\nonumber \\
 & \le2^{h-1}\mathbb{E}_{M,M'}\Bigg[\sup_{\Gamma\in\mathcal{F}_{\tau,\gamma}}\bigg(\bigg|\sum_{u=1}^{m}\sum_{i=1}^{n}\xi_{u,i}\Big[M_{u,i}\log\frac{\sigma(\Gamma_{u,i})}{\sigma(0)}+(1-M_{u,i})\log\frac{1-\sigma(\Gamma_{u,i})}{1-\sigma(0)}\Big]\bigg|^{h}\nonumber \\
 & \qquad\qquad\;\,\quad\quad\qquad\quad+\bigg|\sum_{u=1}^{m}\sum_{i=1}^{n}\xi_{u,i}\Big[M_{u,i}'\log\frac{\sigma(\Gamma_{u,i})}{\sigma(0)}+(1-M_{u,i}')\log\frac{1-\sigma(\Gamma_{u,i})}{1-\sigma(0)}\Big]\bigg|^{h}\bigg)\Bigg] \nonumber \\
 & \le2^{h-1}\Bigg(\mathbb{E}_{M,M'}\bigg[\sup_{\Gamma\in\mathcal{F}_{\tau,\gamma}}\bigg|\sum_{u=1}^{m}\sum_{i=1}^{n}\xi_{u,i}\Big[M_{u,i}\log\frac{\sigma(\Gamma_{u,i})}{\sigma(0)}+(1-M_{u,i})\log\frac{1-\sigma(\Gamma_{u,i})}{1-\sigma(0)}\Big]\bigg|^{h}\bigg]\nonumber \\
 & \quad\qquad\quad+\mathbb{E}_{M,M'}\bigg[\sup_{\Gamma\in\mathcal{F}_{\tau,\gamma}}\bigg|\sum_{u=1}^{m}\sum_{i=1}^{n}\xi_{u,i}\Big[M_{u,i}'\log\frac{\sigma(\Gamma_{u,i})}{\sigma(0)}+(1-M_{u,i}')\log\frac{1-\sigma(\Gamma_{u,i})}{1-\sigma(0)}\Big]\bigg|^{h}\bigg]\Bigg)\nonumber \\
 & =2^{h}\mathbb{E}_{M}\bigg[\sup_{\Gamma\in\mathcal{F}_{\tau,\gamma}}\bigg|\sum_{u=1}^{m}\sum_{i=1}^{n}\xi_{u,i}\Big[M_{u,i}\log\frac{\sigma(\Gamma_{u,i})}{\sigma(0)}+(1-M_{u,i})\log\frac{1-\sigma(\Gamma_{u,i})}{1-\sigma(0)}\Big]\bigg|^{h}\bigg].\label{eq:P-est-helper4}
\end{align}
At this point, we use a contraction argument. As a reminder,
\[
L_{\gamma}:=\sup_{x\in[-\gamma,\gamma]}\frac{|\sigma'(x)|}{\sigma(x)(1-\sigma(x))}.
\]
Thus, for any $x\in[-\gamma,\gamma]$,
\[
\Big|\frac{d}{dx}\log\frac{\sigma(x)}{\sigma(0)}\Big|=\Big|\frac{d}{dx}\log \sigma(x)\Big|=\Big|\frac{1}{\sigma(x)}\sigma'(x)\Big|\le\Big|\frac{\sigma'(x)}{\sigma(x)(1-\sigma(x))}\Big|\le L_{\gamma},
\]
i.e., $x\mapsto\log\frac{\sigma(x)}{\sigma(0)}$ is $L_{\gamma}$-Lipschitz.
A similar argument can be used to justify that $x\mapsto\log\frac{1-\sigma(x)}{1-\sigma(0)}$
is $L_{\gamma}$-Lipschitz over $[-\gamma,\gamma]$. Hence, $x\mapsto\frac{1}{L_{\gamma}}\log\frac{\sigma(x)}{\sigma(0)}$
and $x\mapsto\frac{1}{L_{\gamma}}\log\frac{1-\sigma(x)}{1-\sigma(0)}$ are both
contractions (i.e., 1-Lipschitz). Applying the second inequality of
Theorem 11.6 by \citet{boucheron2013concentration} (with $\Psi(x):=x^{h}$)
and defining $\overline{M}_{u,i}=2M_{u,i}-1\in\{\pm1\}$,
\begingroup
\allowdisplaybreaks
\begin{align}
 & \mathbb{E}_{M}\bigg[\sup_{\Gamma\in\mathcal{F}_{\tau,\gamma}}\bigg|\sum_{u=1}^{m}\sum_{i=1}^{n}\xi_{u,i}\Big[M_{u,i}\log\frac{\sigma(\Gamma_{u,i})}{\sigma(0)}+(1-M_{u,i})\log\frac{1-\sigma(\Gamma_{u,i})}{1-\sigma(0)}\Big]\bigg|^{h}\bigg]\nonumber \\
 & =\mathbb{E}_{M}\bigg[\Psi\bigg(\sup_{\Gamma\in\mathcal{F}_{\tau,\gamma}}\bigg|\sum_{u=1}^{m}\sum_{i=1}^{n}\xi_{u,i}\Big[M_{u,i}\log\frac{\sigma(\Gamma_{u,i})}{\sigma(0)}+(1-M_{u,i})\log\frac{1-\sigma(\Gamma_{u,i})}{1-\sigma(0)}\Big]\bigg|\bigg)\bigg]\nonumber \\
 & =(2L_{\gamma})^{h}\mathbb{E}_{M}\bigg[\Psi\bigg(\frac{1}{2L_{\gamma}}\sup_{\Gamma\in\mathcal{F}_{\tau,\gamma}}\bigg|\sum_{u=1}^{m}\sum_{i=1}^{n}\xi_{u,i}\Big[M_{u,i}\log\frac{\sigma(\Gamma_{u,i})}{\sigma(0)}+(1-M_{u,i})\log\frac{1-\sigma(\Gamma_{u,i})}{1-\sigma(0)}\Big]\bigg|\bigg)\bigg]\nonumber \\
 & \le(2L_{\gamma})^{h}\mathbb{E}_{M}\bigg[\Psi\bigg(\sup_{\Gamma\in\mathcal{F}_{\tau,\gamma}}\bigg|\sum_{u=1}^{m}\sum_{i=1}^{n}\xi_{u,i}\Big[M_{u,i}\Gamma_{u,i}-(1-M_{u,i})\Gamma_{u,i}\Big]\bigg|\bigg)\bigg]\nonumber \\
 & =(2L_{\gamma})^{h}\mathbb{E}_{M}\bigg[\Psi\bigg(\sup_{\Gamma\in\mathcal{F}_{\tau,\gamma}}\bigg|\sum_{u=1}^{m}\sum_{i=1}^{n}\xi_{u,i}\overline{M}_{u,i}\Gamma_{u,i}\bigg|\bigg)\bigg]\nonumber \\
 & =(2L_{\gamma})^{h}\mathbb{E}_{M}\bigg[\sup_{\Gamma\in\mathcal{F}_{\tau,\gamma}}\bigg|\sum_{u=1}^{m}\sum_{i=1}^{n}\xi_{u,i}\overline{M}_{u,i}\Gamma_{u,i}\bigg|^{h}\bigg]\nonumber \\
 & =(2L_{\gamma})^{h}\mathbb{E}_{M}\Big[\sup_{\Gamma\in\mathcal{F}_{\tau,\gamma}}|\langle\Xi\circ\overline{M},\Gamma\rangle|^{h}\Big],\label{eq:P-est-helper5}
\end{align}
\endgroup
where $\Xi\in\{\pm1\}^{m\times n}$ has its $(u,i)$-th entry given
by $\xi_{u,i}$, ``$\circ$'' denotes the Hadamard product, and
$\langle\cdot,\cdot\rangle$ denotes the trace inner product.

Next, we use the result that $|\langle A,B\rangle|\le\|A\|_{2}\|B\|_{*}$,
so
\begin{align}
\mathbb{E}_{M}\Big[\sup_{\Gamma\in\mathcal{F}_{\tau,\gamma}}|\langle\Xi\circ\overline{M},\Gamma\rangle|^{h}\Big] & \le\mathbb{E}_{M}\Big[\sup_{\Gamma\in\mathcal{F}_{\tau,\gamma}}\|\Xi\circ\overline{M}\|_{2}^{h}\|\Gamma\|_{*}^{h}\Big]\nonumber \\
 & \le\mathbb{E}_{M}\Big[\sup_{\Gamma\in\mathcal{F}_{\tau,\gamma}}\|\Xi\circ\overline{M}\|_{2}^{h}(\alpha\sqrt{rmn})^{h}\Big]\nonumber \\
 & =(\tau\sqrt{mn})^{h}\mathbb{E}_{M}[\|\Xi\circ\overline{M}\|_{2}^{h}].\label{eq:P-est-helper6}
\end{align}
Finally, applying Theorem 1.1 of \citet{seginer2000expected}, there
exists a universal constant $C>0$ such that
\begin{align}
\mathbb{E}_{M}[\|\Xi\circ\overline{M}\|_{2}^{h}] & \le C(m^{h/2}+n^{h/2}).\label{eq:P-est-helper7}
\end{align}
In fact, $C=8\cdot2^{1/4}\cdot e^{2}=70.2969\ldots$

At this point, stringing together inequalities \eqref{eq:P-est-helper3}, \eqref{eq:P-est-helper4}, \eqref{eq:P-est-helper5}, \eqref{eq:P-est-helper6}, and \eqref{eq:P-est-helper7}, we get
\begingroup
\allowdisplaybreaks
\begin{align*}
 & \mathbb{E}_{M}\Big[\sup_{\Gamma\in\mathcal{F}_{\tau,\gamma}}|\overline{L}_{M}(\Gamma)-\mathbb{E}_{M}[\overline{L}_{M}(\Gamma)]|^{h}\Big]\\
 & \quad \le\mathbb{E}_{M,M'}\Big[\sup_{\Gamma\in\mathcal{F}_{\tau,\gamma}}|\overline{L}_{M}(\Gamma)-\overline{L}_{M'}(\Gamma)|^{h}\Big]\\
 & \quad \le2^{h}\mathbb{E}_{M}\bigg[\sup_{\Gamma\in\mathcal{F}_{\tau,\gamma}}\bigg|\sum_{u=1}^{m}\sum_{i=1}^{n}\xi_{u,i}\Big[M_{u,i}\log\frac{\sigma(\Gamma_{u,i})}{\sigma(0)}+(1-M_{u,i})\log\frac{1-\sigma(\Gamma_{u,i})}{1-\sigma(0)}\Big]\bigg|^{h}\bigg]\\
 & \quad \le2^{h}(2L_{\gamma})^{h}\mathbb{E}_{M}\Big[\sup_{\Gamma\in\mathcal{F}_{\tau,\gamma}}|\langle\Xi\circ\overline{M},\Gamma\rangle|^{h}\Big]\\
 & \quad \le2^{h}(2L_{\gamma})^{h}(\tau\sqrt{mn})^{h}\mathbb{E}_{M}[\|\Xi\circ\overline{M}\|_{2}^{h}]\\
 & \quad \le2^{h}(2L_{\gamma})^{h}(\tau\sqrt{mn})^{h}C(m^{h/2}+n^{h/2})\\
 & \quad =C(4L_{\gamma}\tau\sqrt{mn})^{h}(m^{h/2}+n^{h/2}).
\end{align*}
\endgroup
Lastly, using the fact that $(a+b)^{p}\le a^{p}+b^{p}$ for $p\in[0,1]$
and $a,b\in\mathbb{R}_{+}$
\begin{align*}
 \Big(\mathbb{E}_{M}\Big[\sup_{\Gamma\in\mathcal{F}_{\tau,\gamma}}|\overline{L}_{M}(\Gamma)-\mathbb{E}_{M}[\overline{L}_{M}(\Gamma)]|^{h}\Big]\Big)^{1/h}
 & \le[C(4L_{\gamma}\tau\sqrt{mn})^{h}(m^{h/2}+n^{h/2})]^{1/h}\\
 & \le C^{1/h}4L_{\gamma}\tau\sqrt{mn}(\sqrt{m}+\sqrt{n}).
\end{align*}
Finally, by choosing $h=\log(m+n)$ and $z=4eL_{\gamma}\tau\sqrt{mn}(\sqrt{m}+\sqrt{n})$
in inequality \eqref{eq:P-est-markov},
\begin{align*}
 \mathbb{P}\Big(\sup_{\Gamma\in\mathcal{F}_{\tau,\gamma}}|\overline{L}_{M}(\Gamma)-\mathbb{E}_{M}[\overline{L}_{M}(\Gamma)]|\ge z\Big)
 & \le\frac{\mathbb{E}_{M}\big[\sup_{\Gamma\in\mathcal{F}_{\tau,\gamma}}|\overline{L}_{M}(\Gamma)-\mathbb{E}_{M}[\overline{L}_{M}(\Gamma)]|^{h}\big]}{z^{h}}.\\
 & \le\frac{[C^{1/\log(m+n)}4L_{\gamma}\tau\sqrt{mn}(\sqrt{m}+\sqrt{n})]^{\log(m+n)}}{[4eL_{\gamma}\tau\sqrt{mn}(\sqrt{m}+\sqrt{n})]^{\log(m+n)}}\\
 & =\frac{C}{m+n}.\qedhere
\end{align*}
\end{proof}

\section{Modifying \textsc{1bitMC} to Allow for Propensity Scores of 1}
\label{sec:1bitMC-modified}

We now discuss how to modify \textsc{1bitMC} along with its theoretical analysis to allow for entries in the propensity score matrix $P$ to be 1. It suffices to make a single change to the algorithm: we replace the feasible set~$\mathcal{F}_{\tau,\gamma}$ in optimization problem~\eqref{eq:stage1-opt} by
\begin{align}
\mathcal{F}_{\tau,\gamma,\varphi}
&:=\big\{ \Gamma\in\mathbb{R}^{m\times n}:\tripleBar \Gamma\tripleBar_{*}\le\tau\sqrt{mn}, \nonumber \\
&\!\!\!\qquad\qquad\qquad\qquad \Gamma_{u,i}\ge-\gamma\text{ for all }u,i, \nonumber \\
&\!\!\!\qquad\qquad\qquad\qquad \Gamma_{u,i}\le\varphi\text{ for all }u,i\text{ s.t.~}M_{u,i}=0 \big\}, \label{eq:new-feasible-set}
\end{align}
where we have introduced a new user-specified parameter $\varphi\in(-\gamma,\gamma)$.
The resulting modified optimization problem is still convex if the original optimization program was convex (which depends on the choice of $\sigma$).
The key idea for the modification is that we allow $\Gamma_{u,i}$ to be as large as possible for entries where $M_{u,i}=1$ (to allow for $\sigma(\Gamma_{u,i})=1$, assuming that $\sigma$ monotonically increases to 1). However, when $M_{u,i}=0$, we enforce that $\Gamma_{u,i}$ cannot be too large. For example, if the $j$-th column is always observed ($M_{u,j}=1$ for all $u$), then there would be no upper bound constraint on any element in the $j$-th column of~$\Gamma$.

For completeness, we present this modified version of \textsc{1bitMC} in Algorithm~\ref{alg:1bitMC-mod}, which we call \textsc{1bitMC-modified}; note that we now intentionally use $\Sigma$ rather than $\sigma$ to denote the link function to avoid confusion as we will take $\sigma$ to be the standard logistic function and $\Sigma$ to a be different function in our theoretical analysis.

\begin{algorithm}[h]
\DontPrintSemicolon
\caption{\textsc{1bitMC-modified}\label{alg:1bitMC-mod}}%
\KwData{Binary matrix $M\in\{0,1\}^{m\times n}$, nuclear norm constraint parameter $\tau>0$, lower bound parameter $\gamma>0$, upper bound parameter $\varphi>-\gamma$, function $\Sigma:\mathbb{R}\rightarrow[0,1]$ (maps real number to probability)}
\KwResult{Estimate $\widehat{P}$ of $P$}
Solve optimization problem~\eqref{eq:stage1-opt} with feasible set $\mathcal{F}_{\tau,\gamma}$ replaced by $\mathcal{F}_{\tau,\gamma,\varphi}$ as given in equation~\eqref{eq:new-feasible-set}.\;
Set $\widehat{P}_{u,i}:=\sigma(\widehat{A}_{u,i})$ for all $u\in[m],i\in[n]$.\;
\end{algorithm}

\subsection*{Theoretical Analysis}

How the theory changes is more involved.
A key theoretical consequence of using feasible set $\mathcal{F}_{\tau,\gamma,\varphi}$ is that we will only be able to accurate estimate entries of $P$ that are in the set $[\Sigma(-\gamma),\Sigma(\varphi)]\cup\{1\}$. We tolerate error in estimating entries of $P$ that are in the ``critical'' interval $(\Sigma(\varphi),1)$ (with $\varphi$ chosen to be sufficiently large, this interval length could be made arbitrarily small). We denote the fraction of entries in $P$ that are in the critical interval as
\[
f_{\text{critical}}(m,n)
:=\frac{1}{mn}\sum_{u=1}^{m}\sum_{i=1}^{n}\ind\{P_{u,i}\in(\Sigma(\varphi),1)\}.
\]
We no longer assume that the true propensity score matrix $P$ is linked to parameter matrix $A$ via the standard logistic function $\sigma$. Instead, we reparameterize $P$ via $P_{u,i}=\Sigma(A_{u,i})$, where
\begin{align}
\Sigma(x) & :=\begin{cases}
\sigma(x) & \text{for }x<-\gamma,\\
\sigma(x)+\underbrace{\frac{1}{2}\Big(1+\frac{x}{\gamma}\Big)(1-\sigma(\gamma))}_{\text{\ensuremath{\substack{\text{linear correction term that is}\\
\text{0 at }x=-\gamma\text{ and }1-\sigma(\gamma)\text{ at }x=\gamma
}
}}} & \text{for }x\in[-\gamma,\gamma],\\
1 & \text{for }x>\gamma.
\end{cases}
\label{eq:fancy-link}
\end{align}
The above choice of $\Sigma$ depends on algorithm parameter $\gamma$. Observe that $A_{u,i}\ge\gamma$ implies that $P_{u,i}=1$ (previously when linking with the standard logistic function, we could not achieve $P_{u,i}=1$ for a finite $A_{u,i}$ value). Our theoretical guarantee for \textsc{1bitMC-modified} depends on the following quantity that summarizes 
Lipschitz smoothness information involving $\log\Sigma$ and $\log(1-\Sigma)$:
\begin{equation*}
\Upsilon_{\gamma,\varphi}:=\max\Big\{1+\frac{1}{2\gamma},\,\frac{1}{1-\Sigma(\varphi)}\Big(\frac{1}{4}+\frac{1}{2\gamma}\Big)\Big\}.
\end{equation*}
Next, we replace Assumption A2 with the following much more general assumption:
\begin{itemize}
\item[\textbf{A2$\bm{'}$.}] There exists some $p_{\min}>0$ such that $P_{u,i}\ge p_{\min}$ for all $u\in[m]$ and $i\in[n]$.
\end{itemize}
We are now ready to state our theoretical guarantee for \textsc{1bitMC-modified}.
\begin{thm}\label{thm:P-est-main-theorem-modified}
Under Assumptions A1, A2$'$, and A3, suppose that we run algorithm \textsc{1bitMC-modified} with $\Sigma$ as given in equation~\eqref{eq:fancy-link}, $\tau\ge\theta$, $\Sigma(-\gamma)\le p_{\min}$, and $\varphi\in(-\gamma,\gamma)$ to obtain the estimate $\widehat{P}$ of propensity score matrix $P$. Let $\widehat{S}\in\mathbb{R}^{m\times n}$ be any matrix satisfying $\tripleBar \widehat{S}\tripleBar_{\max}\le\psi$ for some $\psi\ge\phi$. Let $\delta\in(0,1)$. Then there exists a universal constant $C>0$ such that provided that $m+n\ge C$, with probability at least $1-\frac{C}{m+n}-\delta$ over randomness in which entries are revealed in $X$, we simultaneously have
\begin{align}
\frac{1}{mn}\sum_{u=1}^{m}\sum_{i=1}^{n}(\widehat{P}_{u,i}-P_{u,i})^{2}
& \le 8e\Upsilon_{\gamma,\varphi}\tau\Big(\frac{1}{\sqrt{m}}+\frac{1}{\sqrt{n}}\Big) + \frac{(1-\Sigma(\varphi))^2}{2} f_{\text{critical}}(m,n),
\label{eq:main-inequality1-modified} \\
|L_{\text{IPS-MSE}}(\widehat{S}|\widehat{P})-L_{\text{full MSE}}(\widehat{S})| 
&\le \frac{4\psi^2}{\Sigma(-\gamma)p_{\min}}\Big[\sqrt{8e\Upsilon_{\gamma,\varphi}\tau}\Big(\frac{1}{m^{1/4}}+\frac{1}{n^{1/4}}\Big) \nonumber \\
&\quad\qquad\qquad\quad\;\, + (1 - \Sigma(\varphi)) \sqrt{ \frac{f_{\text{critical}}(m,n)}{2}} \Big] \nonumber \\
&\quad
      + \frac{4\psi^{2}}{p_{\min}}\sqrt{\frac{1}{2mn}\log\frac{2}{\delta}}.
\label{eq:main-inequality2-modified}
\end{align}
\end{thm}
Ignoring terms involving $f_{\text{critical}}(m,n)$, the two bounds~\eqref{eq:main-inequality1-modified} and~\eqref{eq:main-inequality2-modified} are qualitatively similar to their counterparts in Theorem~\ref{thm:P-est-main-theorem} (our main result for \textsc{1bitMC}).
The $f_{\text{critical}}(m,n)$ terms are approximation errors in choosing algorithm parameter $\varphi$ poorly. If there exists some constant $p_{\text{critical}}\in(p_{\text{min}},1)$ such that $P_{u,i}\in[p_{\min},p_{\text{critical}}]\cup\{1\}$ for all $u\in[m],i\in[n]$, and $\varphi\in(-\gamma,\gamma)$ is chosen so that $\Sigma(\varphi)\ge p_{\text{critical}}$, then observe that $f_{\text{critical}}(m,n)=0$. In general, if $f_{\text{critical}}(m,n)$ is nonzero, then we can still have the two error bounds go to 0 provided as $m,n\rightarrow\infty$, we have $f_{\text{critical}}(m,n)\rightarrow0$.

\subsection*{Proof of Theorem~\ref{thm:P-est-main-theorem-modified}}

The theorem is a consequence of the following lemma, which we sketch a proof for at the end of this section.
\begin{lem}\label{lem:P-est-main-theorem-modified}
Under the same assumptions as in Theorem~\ref{thm:P-est-main-theorem-modified}, further assume that there exists $p_{\text{critical}}\in(p_{\text{min}},1)$ such that $P_{u,i}\in[p_{\min},p_{\text{critical}}]\cup\{1\}$ for all $u\in[m],i\in[n]$, and $\varphi\in(-\gamma,\gamma)$ is chosen so that $\Sigma(\varphi)\ge p_{\text{critical}}$. Then there exists a universal constant $C>0$ such that provided that $m+n\ge C$, with probability at least $1-\frac{C}{m+n}$ over randomness in which entries are revealed in $X$, we have
\begin{align}
\frac{1}{mn}\sum_{u=1}^{m}\sum_{i=1}^{n}(\widehat{P}_{u,i}-P_{u,i})^{2}
& \le 4e\Upsilon_{\gamma,\varphi}\tau\Big(\frac{1}{\sqrt{m}}+\frac{1}{\sqrt{n}}\Big). \label{eq:main-inequality1-modified-twice}
\end{align}
\end{lem}
In general, $P$ might not satisfy the additional assumption in Lemma~\ref{lem:P-est-main-theorem-modified}. What we do is consider the projection of $P$ onto matrices that do satisfy the additional assumption. Namely, let the projection $P^{\dagger}\in[0,1]^{m\times n}$ be defined as
\[
P_{u,i}^{\dagger}=\begin{cases}
P_{u,i} & \text{if }P_{u,i}\in[p_{\min},\Sigma(\varphi)]\cup1,\\
\Sigma(\varphi) & \text{if }P_{u,i}\in\big(\Sigma(\varphi),\frac{\Sigma(\varphi)+1}{2}\big],\\
1 & \text{if }P_{u,i}\in\big(\frac{\Sigma(\varphi)+1}{2},1\big].
\end{cases}
\]
Matrix $P^{\dagger}$ is guaranteed to satisfy the conditions on the propensity score matrix in Lemma~\ref{lem:P-est-main-theorem-modified} with $p_{\text{critical}}=\Sigma(\varphi)$.

Next, we have
\begin{align}
\frac1{mn} \tripleBar\widehat{P}-P\tripleBar_{F}^{2}
& \le \frac1{mn}(\tripleBar\widehat{P}-P^{\dagger}\tripleBar_{F}+\tripleBar P^{\dagger}-P\tripleBar_{F})^{2} \nonumber \\
& \le
    \frac2{mn}\tripleBar\widehat{P}-P^{\dagger}\tripleBar_{F}^{2}
    +\frac2{mn}\tripleBar P^{\dagger}-P\tripleBar_{F}^{2}.
\label{eq:P-est-main-theorem-modified-helper1}
\end{align}
We can bound the first RHS term using Lemma~\ref{lem:P-est-main-theorem-modified}:
\begin{equation}
\frac2{mn}\tripleBar\widehat{P}-P^{\dagger}\tripleBar_{F}^{2}
\le 8 e\Upsilon_{\gamma,\varphi}\tau\Big(\frac{1}{\sqrt{m}}+\frac{1}{\sqrt{n}}\Big).
\label{eq:P-est-main-theorem-modified-helper2}
\end{equation}
The second RHS term in inequality~\eqref{eq:P-est-main-theorem-modified-helper1} can be upper-bounded by noticing that the worst-case
absolute entry-wise error of $\frac{1-\Sigma(\varphi)}{2}$ occurs
only for $u\in[m],i\in[n]$ such that $P_{u,i}\in(\Sigma(\varphi),1)$.
Thus,
\begin{align}
 \frac2{mn}\tripleBar P^{\dagger}-P\tripleBar_{F}^{2} 
 & \le\frac2{mn}\Big(\frac{1-\Sigma(\varphi)}{2}\Big)^{2}\sum_{u=1}^{m}\sum_{i=1}^{n}\mathbf{1}\{P_{u,i}\in(\Sigma(\varphi),1)\} \nonumber \\
 & =\frac{(1-\Sigma(\varphi))^{2}}{2}f_{\text{critical}}(m,n).
 \label{eq:P-est-main-theorem-modified-helper3}
\end{align}
Combining inequalities~\eqref{eq:P-est-main-theorem-modified-helper1}, \eqref{eq:P-est-main-theorem-modified-helper2}, and~\eqref{eq:P-est-main-theorem-modified-helper3} yields the theorem's first main bound~\eqref{eq:main-inequality1-modified}. The theorem's second main bound~\eqref{eq:main-inequality2-modified} can then be established using the same proof ideas as in establishing bound~\eqref{eq:main-inequality2-more-general} for Theorem~\ref{thm:P-est-helper}. 

\subsection*{Proof Sketch for Lemma~\ref{lem:P-est-main-theorem-modified}}

\begingroup
\allowdisplaybreaks
We highlight the main change to the proof of bound~\eqref{eq:main-inequality1-more-general} in Theorem~\ref{thm:P-est-helper}. Specifically, we do not assume the condition given by equation~\eqref{eq:weird-L} that defines the variable $L_{\gamma}$ (in fact, as we explain shortly, we replace $L_{\gamma}$ with $\Upsilon_{\gamma,\varphi}$). This affects the contraction argument made in inequality~\eqref{eq:P-est-helper5}.
At the start of inequality~\eqref{eq:P-est-helper5} (for which we replace $\sigma$ with $\Sigma$),
each term of the summation has a factor
\begin{equation}
M_{u,i}\log\frac{\Sigma(\Gamma_{u,i})}{\Sigma(0)}+(1-M_{u,i})\log\frac{1-\Sigma(\Gamma_{u,i})}{1-\Sigma(0)}.
\label{eq:magic-factor}
\end{equation}
Exactly one of the two terms can be nonzero since $M_{u,i}\in\{0,1\}$. Previously, we showed that $x\mapsto\log\frac{\sigma(x)}{\sigma(0)}$ and $x\mapsto\log\frac{1-\sigma(x)}{1-\sigma(0)}$ were each $L_{\gamma}$-Lipschitz for $x\in[-\gamma,\gamma]$ (i.e., $x\mapsto\frac1{L_{\gamma}}\log\frac{\sigma(x)}{\sigma(0)}$ and $x\mapsto\frac1{L_{\gamma}}\log\frac{1-\sigma(x)}{1-\sigma(0)}$ are contractions). Now we show the analogous result using $\Sigma$ instead of $\sigma$ and with the new feasible set $\mathcal{F}_{\tau,\gamma,\varphi}$. There are two cases to consider:

\textbf{Case 1 ($M_{u,i}=1$).} The only possibly nonzero term in expression \eqref{eq:magic-factor} is $\log\frac{\Sigma(\Gamma_{u,i})}{\Sigma(0)}$, where $\Gamma_{u,i}\in[-\gamma,\gamma]$. We
show that the function $x\mapsto\frac{1}{\Upsilon_{\gamma,\varphi}}\log\frac{\Sigma(x)}{\Sigma(0)}$
(for $x\in[-\gamma,\gamma]$) is a contraction by showing that $|\frac{d}{dx}\log\frac{\Sigma(x)}{\Sigma(0)}|\le\Upsilon_{\gamma,\varphi}$. Recall that the standard
logistic function $\sigma$ has $\frac{|\sigma'(x)|}{\sigma(x)(1-x)}=1$. Also, by construction, $\Sigma(x)\ge\sigma(x)$.
We have, for $x\in[-\gamma,\gamma]$,
\begin{align*}
 \Big|\frac{d}{dx}\log\frac{\Sigma(x)}{\Sigma(0)}\Big|
 & = \Big|\frac{d}{dx}\log\Sigma(x)\Big|\\
 & =\frac{1}{\Sigma(x)}\cdot\Big[\sigma'(x)+\frac{1-\sigma(\gamma)}{2\gamma}\Big]\\
 & \le\frac{1}{\sigma(x)}\cdot\Big[\sigma'(x)+\frac{1-\sigma(\gamma)}{2\gamma}\Big]\\
 & \le\frac{\sigma'(x)}{\sigma(x)(1-\sigma(x))}+\frac{1}{\sigma(x)}\cdot\frac{1-\sigma(\gamma)}{2\gamma}\\
 & =1+\frac{1}{\sigma(x)}\cdot\frac{1-\sigma(\gamma)}{2\gamma}\\
 & \le1+\frac{1}{\sigma(-\gamma)}\cdot\frac{1-\sigma(\gamma)}{2\gamma}\\
 & =1+\frac{1}{2\gamma}\\
 & \le \Upsilon_{\gamma,\varphi}.
\end{align*}

\textbf{Case 2 $(M_{u,i}=0$).} The only possibly nonzero term in expression \eqref{eq:magic-factor} is $\log\frac{1-\Sigma(\Gamma_{u,i})}{1-\Sigma(0)}$, where $\Gamma_{u,i}\in[-\gamma,\varphi]$.
We show that the function $x\mapsto\frac{1}{\Upsilon_{\gamma,\varphi}}\log\frac{1-\Sigma(x)}{1-\Sigma(0)}$
(for $x\in[-\gamma,\varphi]$) is a contraction by showing that $|\frac{d}{dx}\log\frac{1-\Sigma(x)}{1-\Sigma(0)}|\le\Upsilon_{\gamma,\varphi}$. Note that $\sigma'(x)\le1/4$ for all $x\in\mathbb{R}$. We have, for $x\in[-\gamma,\varphi]$,
\begin{align*}
 \Big|\frac{d}{dx}\log\frac{1-\Sigma(x)}{1-\Sigma(0)}\Big|
 & =\Big|\frac{d}{dx}\log(1-\Sigma(x))\Big|\\
 & =\Big|\frac{1}{1-\Sigma(x)}\cdot\frac{d}{dx}(1-\Sigma(x))\Big|\\
 & =\frac{1}{1-\Sigma(x)}\cdot\Big[\sigma'(x)+\frac{1-\sigma(\gamma)}{2\gamma}\Big]\\
 & \le\frac{1}{1-\Sigma(\varphi)}\cdot\Big[\sigma'(x)+\frac{1-\sigma(\gamma)}{2\gamma}\Big]\\
 & \le\frac{1}{1-\Sigma(\varphi)}\cdot\Big[\frac{1}{4}+\frac{1}{2\gamma}\Big]\\
 & \le \Upsilon_{\gamma,\varphi}. \qedhere
\end{align*}
\endgroup

\section{More Details on Experiments}
\label{sec:more-experiments}

In this section, we explain why Assumptions A1--A3 hold for the two synthetic datasets (with high probability in the case of \texttt{UserItemData}), and we also present MAE-based results for the numerical experiments on both synthetic and real-world datasets.

\subsection{Sythetic Data}

We verify that Assumptions A1-A3 hold for the synthetic datasets. Assumption A3 holds as both synthetic datasets have partially observed matrix $X$ consist of ratings in a bounded interval. For \texttt{MovieLoverData}, the propensity score matrix $P$ is a block matrix, so it is low-rank, and moreover it has three unique values that are all nonzero and less than 1; thus Assumptions A1 and A2 are both met.
For \texttt{UserItemData}, the propensity score is $P_{u,i} = \sigma(A_{u,i})$ where $A_{u,i}=U_2[u]w_1 + V_2[i]w_2$ where $\sigma$ is the standard logistic function. Hence, parameter matrix $A$ (in Assumptions A1 and A2) has low rank, and in practice after we generate $A$ we can find what its maximum absolute value entry is to satisfy Assumption A2. Alternatively, to obtain a bound that holds with high probability, standard concentration inequality results for the maxima of a finite collection of sub-Gaussian random variables can be used to bound $\|A\|_{\max}$.
In summary, the synthetic datasets we consider satisfy the assumptions of our theoretical analysis.

The MAE-based measures for different algorithms on \texttt{MovieLoverData} and \texttt{UserItemData} are presented in Table~\ref{tab:mcmae}.

\begin{table*}[h]
  \centering
  \begin{tabular}{lllll}
    \toprule
    \multirow{2}{*}{Algorithm} & \multicolumn{2}{c}{\texttt{MovieLoverData}}& \multicolumn{2}{c}{\texttt{UserItemData}}  \\
    \cmidrule(r){2-5}
    ~    &  MAE &  SNIPS-MAE & MAE & SNIPS-MAE  \\
    \midrule
\textsc{PMF} & {0.421  $\pm$  0.013} & {0.421  $\pm$  0.012} & 0.324  $\pm$  0.001 & 0.323  $\pm$  0.001 \\
\textsc{NB-PMF} & 0.490  $\pm$  0.009 & 0.490  $\pm$  0.008 & {\bf 0.308  $\pm$  0.002} & {\bf 0.308  $\pm$  0.002} \\
\textsc{LR-PMF} & N/A & N/A& 0.321  $\pm$  0.002 & 0.325  $\pm$  0.002 \\
\textsc{1bitMC-PMF} & 0.464  $\pm$  0.010 & 0.464  $\pm$  0.010 & {\bf 0.308  $\pm$  0.002} & 0.309  $\pm$  0.002 \\
\arrayrulecolor{black!30}\midrule
\textsc{SVD} & 0.871  $\pm$  0.008 & 0.871  $\pm$  0.008 & 0.310  $\pm$  0.001 & 0.310  $\pm$  0.001 \\
\textsc{NB-SVD}  & 0.722  $\pm$  0.010 & 0.722  $\pm$  0.010 & 0.313  $\pm$  0.001 & 0.314  $\pm$  0.002 \\
\textsc{LR-SVD} & N/A & N/A& 0.312  $\pm$  0.001 & 0.316  $\pm$  0.002 \\
\textsc{1bitMC-SVD} & 0.727  $\pm$  0.009 & 0.727  $\pm$  0.009 & 0.310  $\pm$  0.001 & 0.310  $\pm$  0.001 \\
\arrayrulecolor{black!30}\midrule
\textsc{SVD++} & 0.457  $\pm$  0.011 & 0.457  $\pm$  0.010 & 0.311  $\pm$  0.001 & 0.311  $\pm$  0.001 \\
\textsc{NB-SVD++} & 0.784  $\pm$  0.009 & 0.793  $\pm$  0.009 & 0.318  $\pm$  0.002 & 0.318  $\pm$  0.002 \\
\textsc{LR-SVD++} & N/A & N/A& 0.319  $\pm$  0.001 & 0.323  $\pm$  0.001 \\
\textsc{1bitMC-SVD++} & 0.459  $\pm$  0.011 & 0.459  $\pm$  0.010 & 0.310  $\pm$  0.001 & 0.310  $\pm$  0.001 \\
\arrayrulecolor{black!30}\midrule
\textsc{SoftImpute} & 0.455  $\pm$  0.007 & 0.455  $\pm$  0.007 & 0.550  $\pm$  0.002 & 0.538  $\pm$  0.002 \\
\textsc{NB-SoftImpute} & 0.528  $\pm$  0.006 & 0.527  $\pm$  0.006  & 0.569  $\pm$  0.002 & 0.567  $\pm$  0.003 \\
\textsc{LR-SoftImpute} & N/A & N/A& 0.571  $\pm$  0.002 & 0.562  $\pm$  0.002 \\
\textsc{1bitMC-SoftImpute} & 0.493  $\pm$  0.007 & 0.493  $\pm$  0.006 & 0.557  $\pm$  0.002 & 0.545  $\pm$  0.002 \\
\arrayrulecolor{black!30}\midrule
\textsc{MaxNorm} & 0.571  $\pm$  0.024 & 0.571  $\pm$  0.024 & 0.508  $\pm$  0.002 & 0.496  $\pm$  0.002 \\
\textsc{NB-MaxNorm} & \textbf{0.415  $\pm$  0.043} & \textbf{0.415  $\pm$  0.042} & 0.517  $\pm$  0.006 & 0.507  $\pm$  0.007 \\
\textsc{LR-MaxNorm} & N/A & N/A& 0.520  $\pm$  0.004 & 0.508  $\pm$  0.005 \\
\textsc{1bitMC-MaxNorm} & 0.465  $\pm$  0.042 & 0.465  $\pm$  0.042 & 0.518  $\pm$  0.003 & 0.507  $\pm$  0.003 \\
\arrayrulecolor{black!30}\midrule
\textsc{WTN} & 1.350  $\pm$  0.005 & 1.349  $\pm$  0.005  & 0.527  $\pm$  0.002 & 0.516  $\pm$  0.002 \\
\textsc{NB-WTN} & 1.306  $\pm$  0.019 & 1.305  $\pm$  0.018 & 0.532  $\pm$  0.002 & 0.522  $\pm$  0.002 \\
\textsc{LR-WTN} & N/A & N/A& 0.529  $\pm$  0.002 & 0.519  $\pm$  0.002 \\
\textsc{1bitMC-WTN} & 1.350  $\pm$  0.005 & 1.349  $\pm$  0.005 & 0.527  $\pm$  0.002 & 0.516  $\pm$  0.002 \\
\arrayrulecolor{black!30}\midrule
\textsc{ExpoMF} & 0.547  $\pm$  0.003 & 0.548  $\pm$  0.004  & 0.864  $\pm$  0.005 & 0.889  $\pm$  0.005 \\
\arrayrulecolor{black}\bottomrule
  \end{tabular}
  \smallskip
  \caption{MAE-based metrics of matrix completion methods on synthetic datasets (average $\pm$ standard deviation across 10 experimental repeats).}
  \label{tab:mcmae}
\end{table*}

\subsection{Real-World Data}
The MAE-based measures for different algorithms on \texttt{Coat} and \texttt{MovieLens-100k} are presented in Table~\ref{tab:maereal}.
\begin{table}[h]
  \centering
  \begin{tabular}{lllll}
    \toprule
    \multirow{2}{*}{Algorithm} & \multicolumn{2}{c}{\texttt{Coat}}&\multicolumn{2}{c}{\texttt{MovieLens-100k}}  \\
    \cmidrule(r){2-5}
    ~    & MAE & SNIPS-MAE & MAE & SNIPS-MAE\\
    \midrule
     
\textsc{PMF} & 0.760 &{\bf 0.783}  & 0.736  $\pm$  0.005 & 0.741  $\pm$  0.005 \\
\textsc{NB-PMF} & {0.740} & 0.797 & N/A & N/A\\
\textsc{LR-PMF} & 0.743 & 0.798 & N/A & N/A\\
\textsc{1bitMC-PMF} & 0.759 & {\bf 0.783} & 0.729  $\pm$  0.005 & 0.733  $\pm$  0.005 \\
\arrayrulecolor{black!30}\midrule
\textsc{SVD} & 0.903 & 0.936 & 0.738  $\pm$  0.006 & 0.742  $\pm$  0.006 \\
\textsc{NB-SVD} & 0.879 & 0.947 & N/A & N/A\\
\textsc{LR-SVD} & 0.881 & 0.945 & N/A & N/A\\
\textsc{1bitMC-SVD} & 0.901 & 0.936 & {\bf 0.716  $\pm$  0.005} & {\bf 0.720  $\pm$  0.005} \\
\arrayrulecolor{black!30}\midrule
\textsc{SVD++} & 0.896 & 0.913 & 0.717  $\pm$  0.006 & 0.721  $\pm$  0.006 \\
\textsc{NB-SVD++} &  0.927 & 0.999 & N/A & N/A\\
\textsc{LR-SVD++} & 0.916 & 0.984 & N/A & N/A\\
\textsc{1bitMC-SVD++} & 0.895 & 0.915 & 0.718  $\pm$  0.004 & 0.722  $\pm$  0.004 \\
\arrayrulecolor{black!30}\midrule
\textsc{SoftImpute} & 0.759 & 0.821 & 0.756  $\pm$  0.006 & 0.765  $\pm$  0.006 \\
\textsc{NB-SoftImpute} & 0.751 & 0.811 & N/A & N/A\\
\textsc{LR-SoftImpute} & 0.760 & 0.821 & N/A & N/A\\
\textsc{1bitMC-SoftImpute} & \textbf{0.733} & 0.792 & 0.756  $\pm$  0.005 & 0.764  $\pm$  0.006 \\
\arrayrulecolor{black!30}\midrule
\textsc{MaxNorm} & 0.819 & 0.886 & 0.749  $\pm$  0.005 & 0.754  $\pm$  0.005 \\
\textsc{NB-MaxNorm} & 0.780 & 0.843 & N/A & N/A\\
\textsc{LR-MaxNorm} & 0.829 & 0.896 & N/A & N/A\\
\textsc{1bitMC-MaxNorm} & 0.801 & 0.865 & 0.757  $\pm$  0.007 & 0.764  $\pm$  0.007 \\
\arrayrulecolor{black!30}\midrule
\textsc{WTN} & 0.894 & 0.967 & 0.765  $\pm$  0.005 & 0.770 $\pm$  0.005 \\
\textsc{NB-WTN} & 0.900 & 0.972 & N/A & N/A\\
\textsc{LR-WTN} & 0.891 & 0.963 & N/A & N/A\\
\textsc{1bitMC-WTN} & 0.894 & 0.967 & 0.763  $\pm$  0.005 & 0.768  $\pm$  0.005 \\
\arrayrulecolor{black!30}\midrule
\textsc{ExpoMF} & 1.071 & 1.158 & 1.195  $\pm$  0.023 & 1.223  $\pm$  0.023 \\
\arrayrulecolor{black}\bottomrule
  \end{tabular}
  \vspace{0.5em}
  \caption{MAE-based metrics of matrix completion methods on \texttt{Coat} and \texttt{MovieLens-100k} (results for \texttt{MovieLens-100k} are the averages $\pm$ standard deviations across 10 experimental repeats).}
  \label{tab:maereal}
\end{table}

\end{document}